\newcount\Comments  
\documentclass[letterpaper,10pt]{scrartcl}

\usepackage[utf8]{inputenc}
\usepackage[english]{babel}
\usepackage[automark]{scrlayer-scrpage}
\usepackage{amsmath,amssymb,upref}
\usepackage{amsfonts}
\usepackage{amsthm,array,makecell}
\usepackage{color}
\usepackage{graphicx}
\usepackage{sidecap}
\usepackage{multirow}
\usepackage{booktabs}
\usepackage{fullpage}
\usepackage[title]{appendix}

\usepackage{tikz}
\usetikzlibrary{shapes}
\usetikzlibrary{positioning}
\usetikzlibrary{matrix}
\usetikzlibrary{patterns}
\usetikzlibrary{automata, positioning}

\usepackage[font=small]{caption}
\usepackage{subcaption}

\usepackage{longtable}
\usepackage{rotating}

\usepackage{algorithm}
\usepackage{algcompatible}
\usepackage[noend]{algpseudocode}
\extrafloats{100}
\usepackage{lineno}
\usepackage[hidelinks]{hyperref}

\graphicspath{{figures/}}
\DeclareGraphicsExtensions{.pdf,.eps,.png,.jpg,.jpeg}

\newcommand{\R}{\mathbb{R}}

\newcommand{\bfXi}{\boldsymbol \Xi}

\newcommand{\bfd}{\boldsymbol d}
\newcommand{\bfx}{\boldsymbol x}
\newcommand{\bfX}{\boldsymbol X}

\newcommand{\bff}{\boldsymbol f}

\newcommand{\bfy}{\boldsymbol y}

\newcommand{\Dcal}{\mathcal{D}}

\newcommand{\Vcal}{\mathcal{V}}

\newcommand{\bfeta}{\boldsymbol \eta}

\newcommand{\bfe}{\boldsymbol e}
\newcommand{\bfo}{\boldsymbol o}

\newcommand{\bfu}{\boldsymbol u}

\newcommand{\bfv}{\boldsymbol v}

\newcommand{\bfz}{\boldsymbol z}

\newcommand{\bfA}{\boldsymbol A}
\newcommand{\bfB}{\boldsymbol B}

\newcommand{\bfD}{\boldsymbol D}
\newcommand{\bfE}{\boldsymbol E}

\newcommand{\bfG}{\boldsymbol G}

\newcommand{\bfI}{\boldsymbol I}

\newcommand{\bfM}{\boldsymbol M}
\newcommand{\bfO}{\boldsymbol O}
\newcommand{\bfP}{\boldsymbol P}
\newcommand{\bfR}{\boldsymbol R}
\newcommand{\bfU}{\boldsymbol U}
\newcommand{\bfV}{\boldsymbol V}
\newcommand{\bfN}{\boldsymbol N}

\newcommand{\bfS}{\boldsymbol S}

\newcommand{\bfY}{\boldsymbol Y}
\newcommand{\bfZ}{\boldsymbol Z}

\newcommand{\bfxi}{\boldsymbol \xi}

\newcommand{\singValMax}{s_{\text{max}}}
\newcommand{\singValMin}{s_{\text{min}}}

\newcommand{\nh}{N}
\newcommand{\nr}{n}

\newcommand{\bbfX}{\breve{\bfX}}
\newcommand{\hbfA}{\hat{\bfA}}
\newcommand{\hbfB}{\hat{\bfB}}

\newcommand{\hbfx}{\hat{\bfx}}

\newcommand{\bbfx}{\breve{\bfx}}
\newcommand{\tbfA}{\tilde{\bfA}}
\newcommand{\tbfB}{\tilde{\bfB}}

\newcommand{\tbfx}{\tilde{\bfx}}

\newcommand{\tbfO}{\tilde{\bfO}}
\newcommand{\tbfo}{\tilde{\bfo}}
\newcommand{\hbfO}{\hat{\bfO}}
\newcommand{\hbfo}{\hat{\bfo}}

\newcommand{\bbfY}{\breve{\bfY}}

\newcommand{\rbfd}{\bar{\bfd}}

\newcommand{\Exp}{\mathbb{E}}

\newcommand{\SelectKron}{\bfS}

\newcommand{\tbfz}{\tilde{\bfz}}
\newcommand{\bbfy}{\breve{\bfy}}

\newcommand{\bbfZ}{\breve{\bfZ}}
\newcommand{\bbfz}{\breve{\bfz}}

\newcommand{\kibitz}[2]{\ifnum\Comments=1\textcolor{#1}{#2}\fi}

\newenvironment{keywords}%
   {\begin{trivlist}\item[]{\bfseries\sffamily Keywords:}\ }
   {\end{trivlist}}

\ihead[]{}
\ohead[]{}
\ifoot[]{}
\ofoot[]{}

\newtheorem{theorem}{Theorem}

\newtheorem{corollary}[theorem]{Corollary}

\newtheorem{lemma}[theorem]{Lemma}

\newtheorem{proposition}[theorem]{Proposition}

\numberwithin{equation}{section}

\title{Active operator inference for learning low-dimensional dynamical-system models from noisy data}

\author{Wayne Isaac Tan Uy, Yuepeng Wang, Yuxiao Wen, and Benjamin Peherstorfer\thanks{\{wayne.uy,pehersto\}@cims.nyu.edu, \{yw3114,yw3210\}@nyu.edu, Courant Institute of Mathematical Sciences, New York University, New York, NY 10012}}
\date{July 2021}

\begin{document}

\maketitle

\begin{abstract}

Noise poses a challenge for learning dynamical-system models because already small variations can distort the dynamics described by  trajectory data. This work builds on operator inference from scientific machine learning to infer low-dimensional models from high-dimensional state trajectories polluted with noise. The presented analysis shows that, under certain conditions, the inferred operators are unbiased estimators of the well-studied projection-based reduced operators from traditional model reduction. Furthermore, the connection between operator inference and projection-based model reduction enables bounding the mean-squared errors of predictions made with the learned models with respect to traditional reduced models. The analysis also motivates an active operator inference approach that judiciously samples high-dimensional trajectories with the aim of achieving a low mean-squared error by reducing the effect of noise. Numerical experiments with high-dimensional linear and nonlinear state dynamics demonstrate that predictions obtained with active operator inference have orders of magnitude lower mean-squared errors than operator inference with traditional, equidistantly sampled trajectory data.

\end{abstract}


\section{Introduction}
Noise poses a challenge for learning dynamical-system models because already small variations can distort the dynamics described by  trajectory data. In this work, we build on operator inference \cite{paper:PeherstorferW2016} from scientific machine learning to derive low-dimensional dynamical-system models from high-dimensional, noisy state trajectories. We introduce a sampling scheme to query the high-dimensional systems for data so that, under certain conditions, in particular if the high-dimensional system dynamics are polynomially nonlinear, the inferred operators are unbiased estimators of the well-studied reduced operators obtained via projection of the governing equations of the high-dimensional systems in classical model reduction \cite{book:Antoulas2005,RozzaPateraSurvey,paper:BennerGW2015}. Additionally, we show that the mean-squared error (MSE) of the states predicted with the learned models can be bounded independently of the dimensions of the high-dimensional systems and in terms of the noise-to-signal ratio of the trajectory data. Motivated by the analysis, we propose active operator inference that queries high-dimensional systems in a principled way to generate data with low noise-to-signal ratios, which reduces by a factor of up to three the number of data samples that are required from the high-dimensional systems to make accurate state predictions in our numerical experiments. For the same number of data samples, active operator inference achieves orders of magnitude lower MSEs than traditional, equidistant-in-time sampled trajectory data.

Learning models from data is an active research topic in the field of scientific machine learning. A prominent approach is to fit dynamical-system models to data via dynamic mode decomposition and Koopman-based methods \cite{paper:Schmid2010,FLM:6837872,Tu2014391,NathanBook,Williams2015,doi:10.1137/20M1338289}. In another research direction, sparse representations of governing equations are sought with tools from sparse regression and compressive sensing \cite{paper:BruntonPK2016,Schaeffer6634,Schaeffer2018,Rudye1602614}. There is also work on non-intrusive model reduction that learns coefficients of low-dimensional representations from data \cite{Hesthaven2018,Guo2018,Guo2019}. If frequency-domain or impulse-response data are available, then data-driven modeling methods from the systems and control community are often used, such as the Loewner approach \cite{ANTOULAS01011986,5356286,Mayo2007634,BeaG12,paper:AntoulasGI2016,paper:GoseaA2018,paper:IonitaA2014}, vector fitting \cite{772353,paper:DrmacGB2015}, and eigensystem realization \cite{doi:10.2514/3.20031,KramerERZ}. 

In terms of learning from noisy data, there is the work \cite{Tran2017} that establishes probabilistic recovery guarantees via compressive sensing of sparse systems. Noise-robust data-driven discovery of governing equations is considered in \cite{doi:10.1098/rspa.2018.0305,Zhang2021} using sparse Bayesian regression. A strategy is proposed to subsample the data utilized in solving the regression problem with the goal of reducing the influence of noise on the learned model. A signal-noise decomposition is pursued in \cite{Rudy2019} in which a neural network is trained to discover the underlying dynamics while simultaneously estimating the noise. In system identification, works such as \cite{LjungBook,VIDYASAGAR2008421,1024346,Benner1999,SIMA2007477,pmlr-v75-simchowitz18a} derive probabilistic error bounds for oftentimes linear models using tools from, e.g., random matrix theory. The effect of the presence of noise and perturbations in frequency-domain data have also been studied in data-driven interpolatory model reduction and Loewner methods \cite{BEATTIE20122916,Lefteriu2010,embree2019pseudospectra,DP19LoewnerNoise}. However, except for the interpolatory model reduction methods, which require frequency-domain data, no low-dimensional models are considered in these works. In contrast, our approach based on operator inference and re-projection \cite{paper:PeherstorferW2016,paper:Peherstorfer2019} aims to learn low-dimensional models that are suited for solving outer-loop applications such as design, control, and inverse problems. Operator inference can learn non-Markovian low-dimensional models \cite{UP21NonMarkovian} and it is also a building block for other learning methods such as lift \& learn introduced in \cite{QIAN2020132401,doi:10.2514/1.J058943}, which comes with a sensitivity analysis with respect to deterministic perturbations in data \cite[Chapter~4.3]{QianThesis}. In \cite{UP20OpInfError}, probabilistic \emph{a posteriori} error bounds for operator-inference models are derived for linear models; however, the bounds only hold when data are free of noise. In the following, we exploit the bridge between data-driven modeling with operator inference and traditional model reduction \cite{book:Antoulas2005,RozzaPateraSurvey,paper:BennerGW2015} to establish probabilistic guarantees for learning from noisy data and to inform in a principled way which data samples to query from the high-dimensional system to reduce the effect of noise on the MSE of state predictions. 

This manuscript is organized as follows. Section~\ref{sec:Preliminaries} discusses preliminaries about learning low-dimensional dynamical-system models from data via operator inference and re-projection. Section~\ref{sec:LearningLowDim} describes the sampling and inference problem for learning models from noisy trajectories with the proposed approach. Then, bounds are derived for the MSE of the inferred operators and of the state predictions with respect to projection-based reduced models from traditional model reduction. A design of experiments approach is proposed in Section~\ref{sec:ActiveOpInf}, which leads to active operator inference that selects data samples to reduce the effect of noise on the MSE of state predictions. Numerical results presented in Section~\ref{sec:NumExp} are in agreement with the analysis: the results indicate that active operator inference learns low-dimensional models with MSEs that are orders of magnitude more accurate than with an uninformed design of experiments.

\section{Preliminaries} \label{sec:Preliminaries}

We review operator inference \cite{paper:PeherstorferW2016} for learning low-dimensional models from data in Section~\ref{subsec:OpInf}.  Section~\ref{subsec:Reproj} describes operator inference with the re-projection data sampling scheme \cite{paper:Peherstorfer2019} to recover projection-based reduced models from data.

\subsection{Learning low-dimensional dynamical-system models from data with operator inference} \label{subsec:OpInf}

 Let $\bfx_1, \dots, \bfx_{K} \in \R^N$ be states at time steps $k = 1, \dots, K$ that are obtained by exciting a dynamical system 
    \begin{equation}
    \bfx_{k + 1} = \bff(\bfx_k, \bfu_k)\,,\qquad k = 0, \dots, K - 1,
    \label{eq:HighDimSys}
    \end{equation}
at the inputs $\bfu_0,\dots,\bfu_{K-1} \in \R^p$ and initial condition $\bfx_0 \in \mathbb{R}^{\nh}$. Let further $\Vcal \subset \R^N$ be a subspace of the $N$-dimensional state space $\R^N$. The subspace $\Vcal$ is spanned by the orthonormal columns of the basis matrix $\bfV = [\bfv_1,\dots,\bfv_n] \in \mathbb{R}^{N \times n}$. For example, the subspace $\Vcal$ can be obtained via principal component analysis applied to sampled state trajectories.

Operator inference introduced in \cite{paper:PeherstorferW2016} learns low-dimensional dynamical-system models with polynomial nonlinear terms that best fit the temporal evolution of the state in the subspace $\Vcal$ with respect to the Euclidean norm in a least-squares sense. Operator inference first projects the high-dimensional states $\bfx_0, \dots, \bfx_K$ onto the subspace $\Vcal$ to obtain the projected states $\bbfx_0, \dots, \bbfx_K$ with $\bbfx_k = \bfV^T\bfx_k \in \R^n$ for $k = 0, \dots, K$ and then solves the least-squares problem
   \begin{align} \label{eq:LSOpInf}
        \min_{\hbfA_1,\dots,\hbfA_{\ell},\hbfB}  \sum_{k=0}^{K-1} \left \Vert \sum\nolimits_{j=1}^{\ell} \hbfA_j \bbfx_k^j + \hbfB \bfu_k - \bbfx_{k+1} \right \Vert_2^2\,,
    \end{align}
    where $\ell \in \mathbb{N}$ is the polynomial order, $\hbfB \in \R^{n \times p}, \hbfA_j \in \R^{n \times n_j}$ with 
    \[
    n_j = \binom{n+j-1}{j}\,,\qquad j = 1, \dots, \ell\,,
    \]
and $\bbfx_k^j$ is obtained for $k = 0, \dots, K$ by forming the Kronecker product $j$ times  $\bbfx_k \otimes \cdots \otimes \bbfx_k$ and retaining only the factors whose components are unique up to permutation~\cite{paper:PeherstorferW2016}.
    
\subsection{Recovering projection-based reduced models from data with operator inference and re-projection} \label{subsec:Reproj}

 The re-projection data-sampling scheme introduced in \cite{paper:Peherstorfer2019} judiciously excites the high-dimensional system \eqref{eq:HighDimSys} to generate a re-projected trajectory $\bbfY = [\bbfy_1, \dots, \bbfy_K] \in \mathbb{R}^{\nr \times K}$. The following description follows the version of re-projection described in \cite{QIAN2020132401}. Let $\bar{\bfX} = [\bar{\bfx}_1, \dots, \bar{\bfx}_K]$ be a matrix where each column contains an $\nh$-dimensional vector. For example, in \cite{paper:Peherstorfer2019,QIAN2020132401}, it is proposed to generate $\bar{\bfX}$ by first querying the high-dimensional system \eqref{eq:HighDimSys} at an initial condition and inputs to sample the trajectory $\bfX = [\bfx_1,\dots,\bfx_K]$ and then setting $\bar{\bfX} = \bfX$. Let now $\bfU = [\bfu_1,\dots,\bfu_K]$ be an input trajectory and let $\bbfX = [\bbfx_1, \dots, \bbfx_K]$ be the projected trajectory obtained as $\bbfX = \bfV^T\bar{\bfX}$ from $\bar{\bfX}$. Re-projection then computes $\bfY = [\bfy_1, \dots, \bfy_K]$ by querying the high-dimensional system
    \[
    \bfy_k = \bff(\bfV\bbfx_k, \bfu_k)\,,\qquad k = 1, \dots, K\,,
    \]
    to obtain $\bbfY = \bfV^T\bfY$. The re-projection scheme can be applied to black-box dynamical systems that can be queried at arbitrary initial conditions in $\R^N$ and inputs in $\R^p$. 
    
As shown in \cite{paper:Peherstorfer2019,QianThesis}, if the high-dimensional system \eqref{eq:HighDimSys} from which data are sampled has polynomial form, i.e., 
    \begin{align} \label{eq:PolySys}
        \bff(\bfx, \bfu) = \sum_{j = 1}^{\ell}\bfA_j\bfx^j + \bfB\bfu\,,
    \end{align}
    and if there are sufficiently many data samples, then the solution of the least-squares problem
    \begin{equation}
    \min_{\hbfA_1, \dots, \hbfA_{\ell}, \hbfB} \bar{J}(\hbfA_1, \dots, \hbfA_{\ell}, \hbfB; \bbfX, \bbfY, \bfU)
    \label{eq:LSProb}
    \end{equation}
    with objective
    \begin{equation}
    \bar{J}(\hbfA_1, \dots, \hbfA_{\ell}, \hbfB; \bbfX, \bbfY, \bfU) = \sum_{k = 1}^K\left\|\sum\nolimits_{i = 1}^{\ell} \hbfA_i\bbfx_k^j + \hbfB\bfu_k - \bbfy_k\right\|_2^2
    \label{eq:LSOpInfObj}
    \end{equation}
    is unique and coincides with the projected operators 
     \begin{equation}
     \begin{aligned}
        \tbfB & = \bfV^T \bfB\,, \\
        \tbfA_j & = \bfV^T \bfA_j \bfS_j (\bfV \otimes \cdots \otimes \bfV) \bfR_j, \qquad j = 1,\dots,\ell,
    \end{aligned}
    \label{eq:intrusiveOperators}
    \end{equation}
    where the matrices $\bfS_j \in \R^{N_j \times N^j}$ and $\bfR_j \in \R^{n^j \times n_j}$ satisfy \begin{align*}
        \bfz^j = \bfS_j(\bfz \otimes \cdots \otimes \bfz), \qquad \tbfz \otimes \cdots \otimes \tbfz = \bfR_j \tbfz^j
    \end{align*}
    for all $\bfz \in \R^{N}, \tbfz \in \R^{n}$ and $j = 1, \dots, \ell$ and the Kronecker is applied $j$ times. Notice that the re-projected trajectory $\bbfY$ enters in the objective in the least-squares problem \eqref{eq:LSProb}, whereas only the projected trajectory $\bbfX$ enters in problem \eqref{eq:LSOpInf}.
    
In traditional model reduction, see, e.g., \cite{book:Antoulas2005,RozzaPateraSurvey,paper:BennerGW2015}, the projected operators $\tbfA_1, \dots, \tbfA_{\ell}$, $\tbfB$ are computed directly by computing the matrix-matrix products in the projection step \eqref{eq:intrusiveOperators}. Thus, such traditional model reduction methods are intrusive in the sense that they require the high-dimensional operators $\bfA_1, \dots, \bfA_{\ell}, \bfB$ either in assembled form or implicitly via matrix-vector products. 

\section{Learning low-dimensional models from noisy data} \label{sec:LearningLowDim}

This work investigates operator inference and re-projection for learning low-dimensional models of noisy dynamical systems,
    \begin{equation}
    \bfx_{k + 1} = \bff(\bfx_k, \bfu_k) + \bfxi_k\,,\qquad k = 0, \dots, K-1\,,
    \label{eq:NoisyHighDimSys}
    \end{equation}
    where $\bfxi_0, \dots, \bfxi_{K - 1}$ represent noise. The random vectors $\bfxi_0,\dots,\bfxi_{K-1}$ are independent and each noise vector $\bfxi_k \sim N(\boldsymbol{0},\sigma^2\bfI)$, for $k = 0, \dots, K - 1$, is an $\nh$-dimensional Gaussian random vector with a diagonal covariance matrix and standard deviation $\sigma > 0$ in all directions. In the following, for ease of exposition, the noisy high-dimensional system \eqref{eq:NoisyHighDimSys} can be queried at any initial condition in $\R^N$ with any input in $\R^p$; however, the space of initial conditions and inputs can be restricted to subsets of $\R^N$ and $\R^p$ if necessary.
    
    Section~\ref{subsec:opinfNoisy} applies operator inference and re-projection to learn low-dimensional models from noisy trajectories and derives conditions under which the inferred operators are unbiased estimators of the projection-based reduced operators. The MSE of the learned low-dimensional operators is quantified in terms of the noise-to-signal ratio. In Section~\ref{subsec:ExpError}, we derive bounds on the bias and the MSE of the predicted states of the system described by the learned low-dimensional model with the learned operators for linear and polynomially nonlinear dynamics, respectively. The bounds scale with respect to the noise-to-signal ratio.

\subsection{Operator inference with re-projection with noisy state trajectories} \label{subsec:opinfNoisy}

\begin{figure}
\centering
\fbox{
\begin{tikzpicture}
 \matrix[matrix of math nodes,column sep=2.5em,row
  sep=4em,cells={nodes={circle,draw,minimum width=3em,inner sep=0pt}},
  column 1/.style={nodes={rectangle,draw=none}},
  column 5/.style={nodes={rectangle,draw=none}},
  ampersand replacement=\&] (m) {\begin{tabular}{l}
    projected \\
    trajectory
\end{tabular}
  \&
  \bbfx_1 \& \bbfx_2\& \bbfx_3\& \cdots \& \bbfx_K\\
 \begin{tabular}{l}
    noisy \\
    re-projected \\
    trajectory
\end{tabular} \& 
  \bbfz_1 \& \bbfz_2 \& \bbfz_3 \& \cdots \& \bbfz_K\\
 };
 \foreach \X in {2,3,4,5}
 {\draw[-latex] (m-1-\X) -- (m-1-\the\numexpr\X+1);
 \ifnum\X=5
 \draw[-latex] (m-1-6) -- (m-2-6) node[midway,left,align=left]{query\\  \eqref{eq:NoisyHighDimSys}};
 \else
 \draw[-latex] (m-1-\X) -- (m-2-\X) node[midway,left]{\begin{tabular}{l}
    query \\
    \eqref{eq:NoisyHighDimSys}
\end{tabular}};
 \fi}
\end{tikzpicture}
}
\caption{Applying re-projection to query the noisy high-dimensional system \eqref{eq:NoisyHighDimSys} leads to unbiased estimators of the projected operators \eqref{eq:intrusiveOperators}, which are the very same operators that are obtained with classical, intrusive model reduction.}
\label{fig:NoisyObs}
\end{figure}
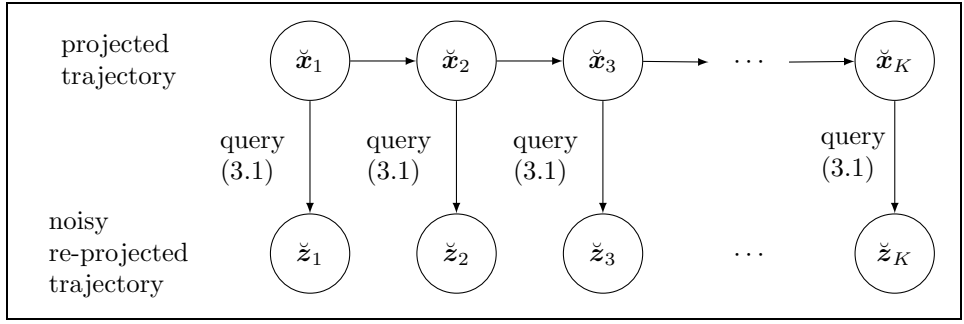

Let $\bar{\bfX}$ be a matrix with $\nh$-dimensional columns (cf.~Section~\ref{subsec:Reproj}) and let $\bfU = [\bfu_1, \dots, \bfu_K]$ be an input trajectory. Note that $\bar{\bfX}$ can also be a realization, a deterministic trajectory, generated by simulating \eqref{eq:NoisyHighDimSys}. We then apply re-projection to obtain $\bfZ = [\bfz_1, \dots, \bfz_K]$ by querying the noisy high-dimensional system \eqref{eq:NoisyHighDimSys} as  
\begin{equation}
\bfz_k = \bff(\bfV\bbfx_k, \bfu_k) + \bfxi_k\,,\qquad k = 1, \dots, K
\label{eq:NoisyReProj}
\end{equation}
where the columns of $\bfXi = [\bfxi_1, \dots, \bfxi_K]$ are independent random noise vectors defined above and $\bbfX = [\bbfx_1, \dots, \bbfx_K] = \bfV^T\bar{\bfX}$ is the projection of $\bar{\bfX}$. The noisy re-projected state trajectory is $\bbfZ = \bfV^T\bfZ = [\bbfz_1,\dots,\bbfz_K]$. 

The corresponding operator-inference problem is
\begin{equation}
\min_{\hbfA_1, \dots, \hbfA_{\ell}, \hbfB} J(\hbfA_1, \dots, \hbfA_{\ell}, \hbfB; \bbfX, \bbfZ, \bfU)
\label{eq:LSQNoisy}
\end{equation}
where the noisy re-projected trajectory $\bbfZ$ enters in the objective \eqref{eq:LSOpInfObj}. To analyze the solution of \eqref{eq:LSQNoisy}, it is beneficial to write \eqref{eq:LSQNoisy} in matrix form as
\begin{equation}
\min_{\bfO} \| \bfD\bfO - \bbfZ^T\|_F^2\,,
\label{eq:LeastSquaresNoisy}
\end{equation}
where the data matrix is $\bfD = [\bbfX^T, (\bbfX^2)^T, \dots, (\bbfX^{\ell})^T, \bfU^T]$ with $\bbfX^i = [\bbfx_1^i, \dots, \bbfx_K^i]$ for $i = 2, \dots, \ell$. The operators $\hbfA_1, \dots, \hbfA_{\ell}, \hbfB$ that we seek are submatrices of $\bfO = [\hbfA_1, \dots, \hbfA_{\ell}, \hbfB]^T$. The size of the data matrix $\bfD$ is $K \times M$ with $M = p + \sum_{j = 1}^{\ell}n_j$. Correspondingly, the size of $\bfO$ is $M \times \nr$.

We now characterize the solution of \eqref{eq:LeastSquaresNoisy} with respect to the noise that is added during the re-projection step. Recall that the procedure to generate $\bbfZ$ is to query the noisy high-dimensional system \eqref{eq:NoisyHighDimSys} at the columns of the projected trajectory $\bbfX$, which is deterministic because $\bar{\bfX}$ is deterministic. Thus, the data matrix $\bfD$ in the regression problem \eqref{eq:LeastSquaresNoisy} is deterministic  while the noisy re-projected trajectory $\bbfZ$ is a random matrix.

Following standard results of least-squares regression, the following proposition summarizes that operator inference together with re-projection leads to an unbiased estimator of the projected operators \eqref{eq:intrusiveOperators} whose variance grows linearly with the variance of the noise.  Additionally, the upper bound of the MSE of the estimator  is controlled by the noise-to-signal ratio $\sigma/\singValMin(\bfD)$, where $\singValMin(\cdot)$ is the minimum singular value of the matrix argument.

    \begin{proposition} \label{prop:LSOperators}
    If $K \geq M$ and $\bfD$ is full rank, then the solution of problem \eqref{eq:LeastSquaresNoisy} is
    \[
    \hbfO = [\hbfA_1, \dots, \hbfA_{\ell}, \hbfB]^T = \tbfO + (\bfD^T\bfD)^{-1}\bfD^T(\bfV^T\bfXi)^T\,,
    \]
    where $\tbfO = [\tbfA_1,\dots,\tbfA_{\ell},\tbfB]^T \in \R^{M \times n}$. In particular, the inferred operators are unbiased estimators of the projection-based reduced operators in the sense that $\Exp[\hbfA_j] = \tbfA_j$ for $j=1,\dots,\ell$ and $\Exp[\hbfB] = \tbfB$. The columns $\hbfo_1, \dots, \hbfo_{\nr}$ of $\hbfO$ are independent random vectors that are distributed as $\hbfo_i \sim N(\tbfo_i, \sigma^2 (\bfD^T \bfD)^{-1})$ for $i = 1, \dots, \nr$ where $\tbfo_1, \dots, \tbfo_{\nr} \in \mathbb{R}^{M}$ are the columns of $\tbfO$. In addition, the MSE is bounded as
    \begin{equation}
    \Exp[\|\hbfO - \tbfO\|_F^2] \le nM \left(\frac{\sigma}{\singValMin(\bfD)} \right)^2.
    \label{eq:MSEBound}
    \end{equation}
    \end{proposition}

    \begin{proof} The following are standard arguments from least-squares regression:     because the data matrix $\bfD$ is full rank and $K \geq M$, the solution of \eqref{eq:LeastSquaresNoisy} is given by the normal equations
     \begin{align*}
        \hbfO = (\bfD^T \bfD)^{-1} \bfD^T \bbfZ^T = \tbfO + (\bfD^T \bfD)^{-1} \bfD^T \breve{\bfXi}^T\,,
    \end{align*}
    where $\breve{\bfXi} = \bfV^T\bfXi$. Since the random vectors $\bfxi_1, \dots, \bfxi_K$ have zero mean, the expectation of $\hbfO$ is $\mathbb{E}_{\xi}[\hbfO] = \tbfO$. Additionally, since $\bfV^T\bfV = \bfI$ is the identity matrix, the entries of $\breve{\bfXi}$ are iid $N(0,\sigma^2)$ random variables which means that the columns of $\breve{\bfXi}^T$ are independent Gaussian random vectors of dimension $K$ with an identity covariance matrix scaled by $\sigma^2$. Thus, the columns of $\hbfO$ are Gaussian with covariance $\sigma^2(\bfD^T\bfD)^{-1}$, which leads to the MSE  
    \[
    \mathbb{E}[\|\hbfO - \tbfO\|_F^2] = \sum_{i = 1}^{\nr} \sum_{j=1}^M \operatorname{Var}[\bfe_j^T \hbfo_i] = \nr \operatorname{tr}( (\bfD^T\bfD)^{-1})\sigma^2 \leq  M\nr\left(\frac{\sigma}{\singValMin(\bfD)}\right)^2,
    \]
    where $\bfe_1,\dots,\bfe_M$ are the canonical basis vectors of $\R^M$. The first equality follows from the unbiasedness of $\hbfO$. 
    \end{proof}

    The independence of the columns of the random matrix $\hbfO$ leads to the independence of the rows of each of the random matrices $\hbfB$ and $\hbfA_j$ for $j=1,\dots,\ell$. However, since the covariance matrix $\sigma^2(\bfD^T \bfD)^{-1}$ of $\hbfo_i^T$ is not necessarily block diagonal, the random matrices $\hbfB$ and $\hbfA_j$ for $j=1,\dots,\ell$ are not necessarily independent.

    In \cite[Chapter~4.3]{QianThesis}, a sensitivity analysis of lift \& learn is presented that just as well applies to operator inference. The analysis leads to bounds with similar right-hand sides as our bound \eqref{eq:MSEBound} on the MSE; however, the analysis in \cite{QianThesis} is restricted to deterministic perturbations and no bounds of the error in the state predictions (as in Section~\ref{subsec:ExpError}) are presented.  
    
\subsection{Error of predicted states with respect to noise-to-signal ratio} \label{subsec:ExpError}

We now consider the random states $\hbfx_1, \dots, \hbfx_K$ predicted by the system described by the learned model
    \begin{align} \label{eq:ROMestimator}
        \hbfx_{k+1} = \sum_{j=1}^{\ell} \hbfA_j \hbfx_k^j + \hbfB \bfu_k, \quad k = 0,\dots,K-1
    \end{align}
with a deterministic initial state $\hbfx_0 \in \mathbb{R}^{\nr}$, which potentially is different from the training initial conditions used to generate the re-projected trajectory. Since the operators $\hbfB, \hbfA_j, j = 1,\dots,\ell$ are random matrices, $\hbfx_k$ is a random vector for $k\ge 1$. In the following, we bound the bias which is the expectation of the difference between the states $\hbfx_1, \dots, \hbfx_K$ and the deterministic states $\tbfx_1, \dots, \tbfx_K$ of the reduced model from intrusive model reduction
    \begin{align} \label{eq:ROM}
        \tbfx_{k+1} & = \sum_{j=1}^{\ell} \tbfA_j \tbfx_k^j + \tbfB \bfu_k, \quad k=0,\dots,K-1,
    \end{align}
    with the operators $\tbfB,\tbfA_j,j=1,\dots,\ell$ defined in \eqref{eq:intrusiveOperators}. Bounds for the MSE between the random states $\hbfx_1, \dots, \hbfx_K$ and the deterministic states $\tbfx_1, \dots, \tbfx_K$ are also deduced.
    
    \subsubsection{Technical preliminaries}
    
It will be useful to account for the difference between the inferred operators and the operators from intrusive model reduction. Let $\bfE_{\hbfA_j}, \bfE_{\hbfB}$ be $n \times n_j$ and $n \times p$ random matrices, respectively, such that 
    \begin{align*}
        \hbfA_j = \tbfA_j + \bfE_{\hbfA_j}, j=1,\dots,\ell, \quad \text{and} \quad \hbfB = \tbfB + \bfE_{\hbfB}.
    \end{align*}
    The distribution of the rows of $\bfE_{\hbfA_j},\bfE_{\hbfB}$ can be described as follows. Define the selection matrices $\bfP_{\bfA_j} \in \R^{n_j \times M}$ for $j=1,\dots,\ell$ and $\bfP_{\bfB} \in \R^{p \times M}$ which satisfy 
    \begin{align*}
        \bfP_{\bfA_j} \hbfO = \hbfA_j^T \quad \text{and} \quad \bfP_{\bfB} \hbfO= \hbfB^T.
    \end{align*}
    For $i=1,\dots,n$, the $i$-th row of $\bfE_{\hbfA_j}$ and $\bfE_{\hbfB}$ are zero-mean multivariate Gaussian random vectors with covariance matrices $ \sigma^2 \Sigma_{\hbfA_j}$ and $\sigma^2 \Sigma_{\hbfB}$, respectively, 
    where $\Sigma_{\hbfA_j} = \bfP_{\bfA_j} (\bfD^T \bfD)^{-1} \bfP_{\bfA_j}^T$ and $\Sigma_{\hbfB} = \bfP_{\bfB} (\bfD^T \bfD)^{-1} \bfP_{\bfB}^T$. Observe that
    \begin{align} \label{eq:NormSigmaB}
         \|\Sigma_{\hbfB}^{1/2}\|_2 = \|\bfP_{\bfB} (\bfD^T \bfD)^{-1} \bfP^T_{\bfB}\|^{1/2}_2 \le \|(\bfD^T \bfD)^{-1}\|^{1/2}_2 = \sqrt{\singValMax((\bfD ^T\bfD)^{-1})} = \frac{1}{\singValMin(\bfD)}
    \end{align}
    where $\singValMax(\cdot)$ is the largest singular value of the matrix argument. Analogously, we have
    \begin{align} \label{eq:NormSigmaA}
        \|\Sigma_{\hbfA_j}^{1/2}\|_2 \le \frac{1}{\singValMin(\bfD)}, \quad j = 1,\dots,\ell.
    \end{align}

The following is a technical lemma derived from  \cite[Theorem 5.32 and Proposition 5.34]{paper:Vershynin2012} that provides an upper bound for the expected value of the powers of the norm of a Gaussian random matrix, which will be utilized in the calculations below; cf.~Appendix~\ref{appendix:GaussNormPowerBoundProof} for the proof.
    
     \begin{lemma}[see, e.g., Theorem 5.32 and Proposition 5.34 in \cite{paper:Vershynin2012}] \label{lemma:GaussNormPowerBound}
        Let $\bfG$ be an $n \times p$ random matrix whose entries are independent standard normal random variables. For $ l \in \mathbb{N}$,
        \begin{align}
            \Exp[\|\bfG\|_2^l] \le (\sqrt{n} + \sqrt{p} + 2^{1/l}\sqrt{l})^l.
        \end{align}
\end{lemma}

\subsubsection{Error in states for linear systems} \label{subsubsec:auto}

 In this section, we consider only systems with $\ell=1$ and therefore drop the subscript in $\bfA_1, \tbfA_1,\hbfA_1$. The operator-inference model is $\hbfx_{k+1} = \hbfA \hbfx_k + \hbfB \bfu_k$ and the model from intrusive model reduction is 
$\tbfx_{k+1} = \tbfA \tbfx_k + \tbfB \bfu_k$.

\begin{proposition} \label{proposition:linearExpErr}
     Let $\hbfx_0 = \tbfx_0$. Suppose that the conditions of Proposition~\ref{prop:LSOperators} hold. If the high-dimensional system \eqref{eq:NoisyHighDimSys} from which data are sampled and the learned low-dimensional model have linear state dependence, for $k \in \mathbb{N}$ with $k \ge 1$, the bias of the state predictions is bounded as
	\begin{align} \label{eq:linearExpErr}
		\|\Exp[\hbfx_k - \tbfx_k]\|_2 \le \sum_{l=2}^k C_l \left(\frac{\sigma}{\singValMin(\bfD)}\right)^l\,,
	\end{align}	     
     where $0 < C_2, \dots, C_k$ are constants that are not functions of $\sigma$ and $\singValMin(\bfD)$. The constants are
     \begin{multline}
     	C_l   = (2\sqrt{n} + 2^{1/l} \sqrt{l})^l \left[ \binom{k}{l} \|\tbfA\|_2^{k-l}\|\tbfx_0\|_2 + \sum_{i=l}^{k-1} \binom{i}{l}  \|\tbfA\|_2^{i-l}\|\tbfB \bfu_{k-1-i}\|_2 \right] \\
     	 + \sum_{i=l-1}^{k-1} \binom{i}{l-1}  \|\tbfA\|^{i-l+1}_2 \|\bfu_{k-1-i}\|_2 \left(2\sqrt{n} + 2^{\frac{1}{2(i-l+1)}} \sqrt{2(i-l+1)}\right)^{i-l+1}(\sqrt{n} + \sqrt{p} + 2),
     \end{multline}
    for $l = 2, \dots, k$.
     \end{proposition}
    
     \begin{proof}
    
 Define the $n \times n$ random matrix $\bfG_{\hbfA}$ as $\bfG_{\hbfA} = \frac{1}{\sigma} \Sigma^{-1/2}_{\hbfA} \bfE_{\hbfA}^T$ and the $p \times n$ random matrix  $\bfG_{\hbfB}$ as $\frac{1}{\sigma} \Sigma_{\hbfB}^{-1/2} \bfE_{\hbfB}^T$. Observe that the entries of $\bfG_{\hbfA}, \bfG_{\hbfB}$ are independent standard random variables.    
      At time step $k$, the solution to the reduced system using the inferred operators is 
     \begin{equation}
     \hbfx_k = \hbfA^k \tbfx_0 + \sum_{i=0}^{k-1} \hbfA^i \hbfB \bfu_{k-1-i}.
     \label{eq:StateAtTimeK}
     \end{equation}

	 We now introduce the following notation: Let $\bfM,\bfN$ be square matrices of the same size. For $m,i \in \mathbb{N}$, denote by $\rho_1(\bfM,\bfN; i,m-i), \dots, \rho_{\binom{m}{i}}(\bfM,\bfN;i,m-i)$ all the $\binom{m}{i}$ possible matrix products with $i$ multiplications of $\bfM$ and $m-i$ multiplications of $\bfN$. For example, if $i=1,m=3$ then $\rho_1(\bfM,\bfN;1,2) = \bfM \bfN^2, \rho_2(\bfM,\bfN;1,2) = \bfN \bfM \bfN$, and $\rho_3(\bfM,\bfN;1,2) = \bfN^2 \bfM$. 
	 
	 We then have with $\hat{\bfA} = \tbfA + \bfE_{\hbfA}$ that 
	 \[
	 \hbfA^k =  \sum_{l=0}^k \sum_{j=1}^{\binom{k}{l}} \rho_j(\tbfA,\bfE_{\hbfA};k-l,l),
	 \]
     which we substitute into \eqref{eq:StateAtTimeK} at time step $k$, to obtain
     \begin{align} \label{eq:xhatTimeStepLin}
         \hbfx_k & =  \sum_{l=0}^k \sum_{j=1}^{\binom{k}{l}} \rho_j(\tbfA,\bfE_{\hbfA};k-l,l) \tbfx_0 + \sum_{i=0}^{k-1} \sum_{l=0}^i \sum_{j=1}^{\binom{i}{l}} \rho_j(\tbfA,\bfE_{\hbfA};i-l,l) \hbfB \bfu_{k-1-i} \notag \\
         & = \sum_{l=0}^k \sum_{j=1}^{\binom{k}{l}} \rho_j(\tbfA,\bfE_{\hbfA};k-l,l) \tbfx_0 + \sum_{l=0}^{k-1} \sum_{i=l}^{k-1} \sum_{j=1}^{\binom{i}{l}} \rho_j(\tbfA,\bfE_{\hbfA};i-l,l) \hbfB \bfu_{k-1-i} \notag \\
         & = \sum_{l=0}^k \sum_{j=1}^{\binom{k}{l}} \rho_j(\tbfA,\bfE_{\hbfA};k-l,l) \tbfx_0 + \sum_{l=0}^{k-1} \sum_{i=l}^{k-1} \sum_{j=1}^{\binom{i}{l}} \rho_j(\tbfA,\bfE_{\hbfA};i-l,l) \tbfB \bfu_{k-1-i} \\
         & \quad \quad
         + \sum_{l=0}^{k-1} \sum_{i=l}^{k-1} \sum_{j=1}^{\binom{i}{l}} \rho_j(\tbfA,\bfE_{\hbfA};i-l,l) \bfE_{\hbfB} \bfu_{k-1-i} \notag
     \end{align}
     where in the second equality, we interchanged the order of the summation for the second term in the sum and in the third equality, we used $\hbfB = \tbfB + \bfE_{\hbfB}$.
      Notice that for the state obtained with intrusive model reduction we have
     \begin{equation}
         \tbfx_k = \sum_{j=1}^{\binom{k}{0}} \rho_j(\tbfA,\bfE_{\hbfA};k,0) \tbfx_0 +  \sum_{i=0}^{k-1} \sum_{j=1}^{\binom{i}{0}} \rho_j(\tbfA,\bfE_{\hbfA};i,0) \tbfB \bfu_{k-1-i}
         \label{eq:xtildetDetTimeStepLin}
     \end{equation}
     which corresponds to the first 2 terms of \eqref{eq:xhatTimeStepLin} but with $l=0$ fixed. Thus, \eqref{eq:xtildetDetTimeStepLin} consists of all terms in \eqref{eq:xhatTimeStepLin} where the random matrices $\bfE_{\hbfA},\bfE_{\hbfB}$ are absent.
      Hence,
     \begin{align*}
         \hbfx_k - \tbfx_k & = \sum_{l=1}^k \sum_{j=1}^{\binom{k}{l}} \rho_j(\tbfA,\bfE_{\hbfA};k-l,l) \tbfx_0 + \sum_{l=1}^{k-1} \sum_{i=l}^{k-1} \sum_{j=1}^{\binom{i}{l}} \rho_j(\tbfA,\bfE_{\hbfA};i-l,l) \tbfB \bfu_{k-1-i} \\
         & \quad \quad
         + \sum_{l=0}^{k-1} \sum_{i=l}^{k-1} \sum_{j=1}^{\binom{i}{l}} \rho_j(\tbfA,\bfE_{\hbfA};i-l,l) \bfE_{\hbfB} \bfu_{k-1-i}. 
     \end{align*}
      Additionally, when $l = 1$, the terms  $\Exp[\rho_j(\tbfA,\bfE_{\hbfA};k-l,l) ]$ and $\Exp[\rho_j(\tbfA,\bfE_{\hbfA};i-l,l) ]$ are zero because $\bfE_{\hbfA}$ has zero mean. Similarly, for $l = 0$, the terms $ \Exp[\rho_j(\tbfA,\bfE_{\hbfA};i,0) \bfE_{\hbfB}] $ are zero.
      This means that 
     \begin{align}\label{eq:LinExpDiscrepThreeTerms}
         \|\Exp[\hbfx_k - \tbfx_k]\|_2 \le \tau_1 + \tau_2 + \tau_3
     \end{align}
     where
     \begin{align*}
         \tau_1 & = \sum_{l=2}^k \sum_{j=1}^{\binom{k}{l}} \| \Exp [\rho_j(\tbfA,\bfE_{\hbfA};k-l,l) \tbfx_0] \|_2, \\
         \tau_2 & = \sum_{l=2}^{k-1} \sum_{i=l}^{k-1} \sum_{j=1}^{\binom{i}{l}} \| \Exp[\rho_j(\tbfA,\bfE_{\hbfA};i-l,l) \tbfB \bfu_{k-1-i}] \|_2, \\
         \tau_3 & = \sum_{l=1}^{k-1} \sum_{i=l}^{k-1} \sum_{j=1}^{\binom{i}{l}} \| \Exp[ \rho_j(\tbfA,\bfE_{\hbfA};i-l,l) \bfE_{\hbfB} \bfu_{k-1-i}] \|_2.
     \end{align*}
      It remains to bound each of $\tau_1,\tau_2,
     \tau_3$.
    Since
    \begin{align*}
    	\|\Exp [\rho_j(\tbfA,\bfE_{\hbfA};k-l,l) \tbfx_0] \|_2 & =\|\Exp[ \rho_j(\tbfA,\sigma \bfG_{\hbfA}^T \Sigma^{1/2}_{\hbfA};k-l,l) ]\tbfx_0 \|_2  \\
    	& \le \sigma^l \|\Sigma_{\hbfA}^{1/2}\|^l_2 \|\tbfA\|_2^{k-l}  \|\tbfx_0\|_2 \Exp [\|\bfG_{\hbfA}\|^l_2] \\
    	& \le \left(\frac{\sigma}{s_{min} (\bfD)} \right)^l \|\tbfA\|_2^{k-l}  \|\tbfx_0\|_2 \Exp [\|\bfG_{\hbfA}\|_2^l],
    \end{align*}
we obtain
     \begin{align*}
         \tau_1  & \le \sum_{l=2}^k \binom{k}{l} \left(\frac{\sigma}{\singValMin (\bfD)} \right)^l \|\tbfA\|_2^{k-l}\|\tbfx_0\|_2 \Exp[\|\bfG_{\hbfA}|\|^l_2] \\
         & \le \sum_{l=2}^k \binom{k}{l} \left(\frac{\sigma}{\singValMin (\bfD)} \right)^l \|\tbfA\|_2^{k-l}\|\tbfx_0\|_2 (2\sqrt{n} + 2^{1/l} \sqrt{l})^l
     \end{align*}
by applying Lemma~\ref{lemma:GaussNormPowerBound}.
      Likewise, 
     \begin{align*}
         \tau_2  & \le \sum_{l=2}^{k-1} \sum_{i=l}^{k-1} \binom{i}{l} \left(\frac{\sigma}{\singValMin (\bfD)} \right)^l \|\tbfA\|_2^{i-l}\|\tbfB \bfu_{k-1-i}\|_2 \Exp[\|\bfG_{\hbfA}|\|^l_2] \\
         & \le  \sum_{l=2}^{k-1} \sum_{i=l}^{k-1} \binom{i}{l} \left(\frac{\sigma}{\singValMin (\bfD)} \right)^l \|\tbfA\|_2^{i-l}\|\tbfB \bfu_{k-1-i}\|_2 (2\sqrt{n} + 2^{1/l} \sqrt{l})^l.
     \end{align*}
      Finally, 
     \begin{align*}
      \| \Exp[ \rho_j(\tbfA,\bfE_{\hbfA};i-l,l) \bfE_{\hbfB} \bfu_{k-1-i}]\|_2 
     & \le \Exp [\| \rho_j(\tbfA,\sigma \bfG_{\hbfA}^T \Sigma_{\hbfA}^{1/2}; i-l,l) \sigma \bfG_{\hbfB}^T \Sigma_{\hbfB}^{1/2} \bfu_{k-1-i} \|_2] \\
     & \le \sigma^{l+1} \|\Sigma_{\hbfA}^{1/2}\|_2^l \|\Sigma_{\hbfB}^{1/2}\|_2 \|\tbfA\|^{i-l}_2 \|\bfu_{k-1-i}\|_2 \Exp[\|\bfG_{\hbfA}\|^{i-l}_2 \|\bfG_{\hbfB}\|_2] \\
     & \le \left( \frac{\sigma}{\singValMin(\bfD)} \right)^{l+1} \|\tbfA\|^{i-l}_2 \|\bfu_{k-1-i}\|_2 \Exp[ \|\bfG_{\hbfA}\|^{i-l}_2 \|\bfG_{\hbfB}\|_2 ]
 \end{align*}
     so that 
     \begin{multline}
         \tau_3  
         \le \sum_{l=2}^{k} \sum_{i=l-1}^{k-1} \binom{i}{l-1} \left( \frac{\sigma}{\singValMin(\bfD)} \right)^{l} \|\tbfA\|^{i-l+1}_2 \|\bfu_{k-1-i}\|_2 \\
         \left(2\sqrt{n} + 2^{\frac{1}{2(i-l+1)}} \sqrt{2(i-l+1)}\right)^{i-l+1}(\sqrt{n} + \sqrt{p} + 2)
     \end{multline}
     because the Cauchy-Schwarz inequality and Lemma~\ref{lemma:GaussNormPowerBound} lead to
     \begin{align*}
      \left \vert \Exp[ \|\bfG_{\hbfA}\|^{i-l+1}_2 \|\bfG_{\hbfB}\|_2 ] \right \vert & \le \sqrt{\Exp[  \|\bfG_{\hbfA}\|^{2(i-l+1)}_2 ] \Exp[  \|\bfG_{\hbfB}\|_2^2 ]} \\
      & \le \left(2\sqrt{n} + 2^{\frac{1}{2(i-l+1)}} \sqrt{2(i-l+1)}\right)^{i-l+1}(\sqrt{n} + \sqrt{p} + 2).
     \end{align*}
      The result follows by combining the upper bounds for $\tau_1,\tau_2,\tau_3$.
    \end{proof} 
    
    \begin{corollary} \label{prop:StateErrExpAuto}
    Let $\hbfx_0 = \tbfx_0$. If $\ell = 1$ and high-dimensional system \eqref{eq:NoisyHighDimSys} from which data are sampled is autonomous, then 
    \begin{align} \label{eq:autoerrorbound}
        \|\Exp[\hbfx_k - \tbfx_k]\|_2 \le \sum_{l=2}^k \binom{k}{l}   \left(\frac{\sigma}{\singValMin(\bfD)}\right)^l (2\sqrt{n} + 2^{1/l} \sqrt{l})^l \|\tbfA\|_2^{k-l} \|\tbfx_0\|_2,
    \end{align}
    for $k \in \mathbb{N}$ with $k \ge 1$. 
    \end{corollary}
    
    \begin{proof}
    	The proof follows that of Proposition~\ref{proposition:linearExpErr} noting that $\hbfB,\tbfB$ and hence $\bfE_{\hbfB}$ are zero matrices.
    \end{proof}
   
Several remarks are in order. As the time step $k$ increases, i.e., as we move forward in time, the bound \eqref{eq:linearExpErr} also increases, which is expected because the bias  of the state estimators of previous time steps is accumulated. Notice that the bound also depends on $n$, the dimension of the reduced space, and on $p$, the dimension of the input. The bound \eqref{eq:linearExpErr} further suggests that if the noise-to-signal ratio $\sigma/\singValMin(\bfD) < 1$, the term associated with $(\sigma/\singValMin(\bfD))^2$ at time step $k = 2$ dominates the upper bound as $\sigma/\singValMin(\bfD) \rightarrow 0$. Hence, we expect that for  $\sigma/\singValMin(\bfD)$ sufficiently small, an order of magnitude decrease in the noise-to-signal ratio yields at least a decrease of 2 orders of magnitude in the bias of the predicted states.

    \subsubsection{Re-sampling operators for unbiased state predictions} 
    
We now devise a strategy to sample \eqref{eq:ROMestimator} for $\ell=1$ that guarantees that predicted states are unbiased, i.e., $\Exp[\hbfx_k] = \tbfx_k$ for all $k \in \mathbb{N}$.
    
    \begin{proposition} \label{prop:indepCopies}
        Suppose that the conditions of Proposition~\ref{prop:LSOperators} hold. For $k=1,\dots,K$, let $\hbfA^{(1)},\dots,\hbfA^{(K)}$ be independent samples of $\hbfA$ such that $\hbfA^{(1)},\dots,\hbfA^{(K)},\hbfB$ are mutually independent. Let the high-dimensional system \eqref{eq:NoisyHighDimSys} from which data are sampled and the learned low-dimensional model be linear. If the reduced state $\hbfx_k$ is computed by integrating the time varying dynamical-system model
        $$\hbfx_{k+1} = \hbfA^{(k+1)}\hbfx_k + \hbfB \bfu_k$$
        for $k=0,\dots,K-1$, then if $\hbfx_0 = \tbfx_0$, the predicted states are unbiased in the sense $\Exp[\hbfx_k] = \tbfx_k, k = 0,\dots,K$.
    \end{proposition}

    \begin{proof}
        From the given assumptions, the operators $\hbfA^{(1)},\dots,\hbfA^{(K)},\hbfB$ are unbiased following Proposition~\ref{prop:LSOperators}. We now proceed via induction. When $k=0$, since $\hbfx_0 = \tbfx_0$, $\hbfx_1 = \hbfA^{(1)}\tbfx_0 + \hbfB \bfu_0$. Therefore, $\Exp[\hbfx_1] = \Exp[\hbfA^{(1)}] \tbfx_0 + \Exp[\hbfB] \bfu_0 = \tbfA \tbfx_0 + \tbfB \bfu_0 = \tbfx_1$. 
        For a time step $j \in \{1,\dots,K-1\}$, suppose that $\Exp[\hbfx_k] = \tbfx_k$ for $k=0,\dots,j-1$. Observe that $\hbfx_{j-1}$ is a function of $\hbfA^{(j-1)},\dots,\hbfA^{(1)},\hbfB,\tbfx_0$. Independence of $\hbfA^{(j)}$ to $\hbfA^{(j-1)},\dots,\hbfA^{(1)},\hbfB$ implies that $\hbfA^{(j)}$ and $\hbfx_{j-1}$ are independent as well. Therefore, 
        \begin{align*}
            \Exp[\hbfx_j] & = \Exp[\hbfA^{(j-1)}\hbfx_{j-1} + \hbfB \bfu_{j-1}]  = \Exp[\hbfA^{(j-1)}] \Exp[\hbfx_{j-1}] + \Exp[\hbfB] \bfu_{j-1} \\
            & = \tbfA \tbfx_{j-1} + \tbfB \bfu_{j-1} = \tbfx_j.
        \end{align*}

    \end{proof}
    
    The random matrices $\hbfA^{(1)},\dots,\hbfA^{(K)}, \hbfB$, which satisfy the conditions stated in Proposition~\ref{prop:indepCopies}, can be generated by solving \eqref{eq:LeastSquaresNoisy} $K+1$ times, each time with a new, mutually independent, re-projected noisy state trajectories. Generating $\hbfA^{(1)},\dots,\hbfA^{(K)}$ can be computationally expensive because the high-dimensional system needs to be queried for a potentially large number of trajectories.

\subsubsection{Error of state predictions with polynomially nonlinear systems} \label{subsubsec:NonlinearPoly}
    
We now derive bounds for the bias $\|\Exp[\hbfx_k - \tbfx_k]\|_2$ and the MSE $\Exp[\|\hbfx_k - \tbfx_k\|_2^2]$ of state predictions where data are sampled from polynomially nonlinear systems. 

We start by writing the state $\hbfx_k$ at time step $k$ (which only involves the initial condition and previous inputs) as a sum of vectors, each of which is formed as a combination of matrix and Kronecker products.
    
    \begin{lemma} \label{lemma:TimeStepNonLin}
           The state $\hbfx_k$ at time step $k, k \in \mathbb{N},$  of the polynomially nonlinear model \eqref{eq:ROMestimator} is 
        \begin{align} \label{eq:timeStepExpr}
            \hbfx_k = \sum_{l = 0}^{Q_k} \sum_{\substack{j_1,\dots,j_{\ell +1} \in \mathbb{N} \\j_{1}+\dots+j_{\ell + 1} = l}} \zeta_{j_1,\dots,j_{\ell + 1}}(\bfE_{\hbfA_1},\dots,\bfE_{\hbfA_{\ell}},\bfE_{\hbfB}),
        \end{align}
        where $Q_k = (\ell^k-1)/(\ell-1)$ and $\zeta_{j_1,\dots,j_{\ell + 1}}$ is a sum of vectors in $\R^n$, where each term is a combination of matrix and Kronecker products involving the deterministic quantities $\hbfx_0,\bfu_0,\dots,\bfu_{k-1},\tbfA_1,\dots,\tbfA_{\ell},\tbfB$ and the random matrices $\bfE_{\hbfA_1},\dots,\bfE_{\hbfA_{\ell}},\bfE_{\hbfB}$. Each term in the sum $\zeta_{j_1,\dots,j_{\ell + 1}}$ consists of $j_l$ multiplications of $\bfE_{\hbfA_l}$ for $l = 1,\dots,\ell$ and $j_{\ell+1}$ multiplications of $\bfE_{\hbfB}$.
    \end{lemma}
    
    \begin{proof}
        We recast the system \eqref{eq:ROMestimator} as
    \begin{align} \label{eq:ROMestimatorKron}
        \hbfx_{k+1} = \sum_{j=1}^{\ell} \hbfA_j \SelectKron_j (\hbfx_k \otimes \cdots \otimes \hbfx_k) + \hbfB \bfu_k,
    \end{align}
    where $\SelectKron_j \in \R^{n_j \times n^j}$ is a selection matrix such that $\SelectKron_j (\hbfx_k \otimes \cdots \otimes \hbfx_k) = \hbfx_k^j$ for $j=1,\dots,\ell$. Thus, the state $\hbfx_{k + 1}$ at time step $k + 1$ is the result of recursively applying \eqref{eq:ROMestimatorKron} until the right-hand side contains only the initial condition $\hbfx_0$ and the inputs $\bfu_0,\dots,\bfu_{k-1}$. Since $\hbfA_j = \tbfA_j + \bfE_{\hbfA_j}$ for $j=1,\dots,\ell$ and $\hbfB = \tbfB + \bfE_{\hbfB}$, the right-hand side is a sum of combinations of matrix and Kronecker products involving the deterministic vectors $\hbfx_0,\bfu_0,\dots,\bfu_{k-1}$, the deterministic matrices $\tbfA_1,\dots,\tbfA_{\ell},\tbfB, \bfS_1,\dots,\bfS_{\ell}$ and the random matrices $\bfE_{\hbfA_1},\dots,\bfE_{\hbfA_{\ell}},\bfE_{\hbfB}$. The right-hand side can then be ordered with respect to the number of times there is a multiplication with a random matrix which is represented by the outer sum in \eqref{eq:timeStepExpr}. The outer sum is further partitioned according to how often there is a multiplication involving $\bfE_{\hbfA_1},\dots,\bfE_{\hbfA_{\ell}},\bfE_{\hbfB}$ with corresponding frequencies of multiplications $j_1,\dots,j_{\ell+1}$ times. The frequencies $j_1,\dots,j_{\ell+1}$  serve as  indices of the inner sum \eqref{eq:timeStepExpr}. 
    
    It remains to show that the state at time step $k$ is obtained using at most $Q_k$ multiplications with a random matrix. We proceed via induction. When $k=1$,
    \begin{align*}
        \hbfx_{1} = \sum_{j=1}^{\ell} \tbfA_j \SelectKron_j (\hbfx_0 \otimes \cdots \otimes \hbfx_0) + \tbfB \bfu_0 + \sum_{j=1}^{\ell} \bfE_{\hbfA_1} \SelectKron_j (\hbfx_0 \otimes \cdots \otimes \hbfx_0) + \bfE_{\hbfB} \bfu_0,
    \end{align*}
    thereby implying that there is at most one ($Q_1=1$) random-matrix multiplication to obtain $\hbfx_1$. Suppose that at time step $k=m$, obtaining the state $\hbfx_m$ requires at most $Q_m$ random-matrix multiplications. At time step $k=m+1$, the maximum number of random-matrix multiplications is determined by the expression $\bfE_{\hbfA_{\ell}} \bfS_{\ell} (\hbfx_m \otimes \cdots \otimes \hbfx_m)$. From the induction step, $\hbfx_m$ has at most $Q_m = (\ell^m-1)/(\ell-1)$ random matrix multiplications which means that $\hbfx_{m+1}$ has at most $Q_m\ell + 1 =(\ell^{m+1}-1)/(\ell-1)=Q_{m+1}$ random matrix multiplications due to the $\ell$ Kronecker products of $\hbfx_m$ and the random matrix $\bfE_{\hbfA_{\ell}}$.
    \end{proof}
    
The following proposition shows that the bound for the bias $\|\Exp[\hbfx_k - \tbfx_k]\|_2$ of state predictions is still polynomial in terms of the noise-to-signal ratio even when data are sampled from polynomially nonlinear systems and polynomially nonlinear models are learned. In particular, when $\sigma/\singValMin(\bfD) < 1$ and $\sigma/\singValMin(\bfD) \rightarrow 0$, the behavior of the upper bound is dominated by the term associated with $(\sigma/\singValMin(\bfD))^2$.
     
\begin{proposition} \label{prop:nonlinearpoly}
         Let $\hbfx_0 = \tbfx_0$. Suppose that the conditions of Proposition~\ref{prop:LSOperators} hold. If the high-dimensional system \eqref{eq:NoisyHighDimSys} from which data are sampled and the learned low-dimensional model are polynomially nonlinear, for $k \in \mathbb{N}$ with $k \ge 1$, it holds that
    $$\|\Exp[\hbfx_k - \tbfx_k]\|_2 \le \sum_{l=2}^{Q_k} \bar{C}_l \left(\frac{\sigma}{\singValMin(\bfD)}\right)^l $$
    for some constants $0 < \bar{C}_l < \infty, l =2,\dots,Q_k$, which are not functions of $\sigma$ and $\bfD$.     
     \end{proposition}
     
    \begin{proof}
    Define the $n_j \times n$ random matrix $\bfG_{\hbfA_j}$ as $\bfG_{\hbfA_j} = \frac{1}{\sigma} \Sigma_{\hbfA_j}^{-1/2} \bfE_{\hbfA_j}^T$ for $j=1,\dots,\ell$ and the $p \times n$ random matrix $\bfG_{\hbfB}$ as $\bfG_{\hbfB} = \frac{1}{\sigma} \Sigma_{\hbfB}^{-1/2} \bfE_{\hbfB}^T$. Observe that the entries of $\bfG_{\hbfA_j}$ for $j=1,\dots,\ell$ and $\bfG_{\hbfB}$ are independent standard normal random variables, however, in general, the random matrices are dependent due to the dependence between $\hbfB,\hbfA_j, j =1,\dots,\ell$.
    According to Lemma~\ref{lemma:TimeStepNonLin}, the state $\hbfx_k$ at time step $k$ is
    \begin{align} \label{eq:NonlinxHatTimeStep}
        \hbfx_k = \sum_{l = 0}^{Q_k} \sum_{\substack{j_1,\dots,j_{\ell +1} \in \mathbb{N} \\j_{1}+\dots+j_{\ell + 1} = l}} \zeta_{j_1,\dots,j_{\ell + 1}}(\bfE_{\hbfA_1},\dots,\bfE_{\hbfA_{\ell}},\bfE_{\hbfB}).
    \end{align}
   Since the system is polynomially nonlinear, the state obtained with intrusive model reduction, 
        \begin{align}\label{eq:NonlinxTildeTimeStep}
            \tbfx_k =  \zeta_{0,\dots,0}(\bfE_{\hbfA_1},\dots,\bfE_{\hbfA_{\ell}},\bfE_{\hbfB})
        \end{align}
    is comprised of those terms for which no random matrix is present in the multiplications ($l=0$). 
   We now isolate the terms in \eqref{eq:NonlinxHatTimeStep} in which a single random matrix is involved in the multiplication. By linearity of expectation and using that the random matrices $\bfE_{\hbfB}, \bfE_{\hbfA_{s}}, s = 1,\dots,\ell$ have zero mean,
    \begin{align*}
            \Exp \left[  \sum_{\substack{j_{1},\dots,j_{\ell+1}\\j_{1}+\dots+j_{\ell+1} = 1}} \zeta_{j_1,\dots,j_{\ell+1}}(\bfE_{\hbfA_1},\dots,\bfE_{\hbfA_{\ell}},\bfE_{\hbfB}) \right] = \boldsymbol{0}.
    \end{align*}
    Therefore, with the triangle and Jensen's inequality follows
    \begin{equation}
    \|\Exp[\hbfx_k - \tbfx_k]\|_2 \leq \sum_{l = 2}^{Q_k} \sum_{\substack{j_1,\dots,j_{\ell +1} \in \mathbb{N} \\j_{1}+\dots+j_{\ell + 1} = l}}  \Exp \left[ \| \zeta_{j_1,\dots,j_{\ell+1}}(\bfE_{\hbfA_1},\dots,\bfE_{\hbfA_{\ell}},\bfE_{\hbfB}) \|_2 \right]\,.\label{eq:ProofInEqBias}
    \end{equation}
    It remains to bound $\Exp \left[ \| \zeta_{j_1,\dots,j_{\ell+1}}(\bfE_{\hbfA_1},\dots,\bfE_{\hbfA_{\ell}},\bfE_{\hbfB}) \|_2 \right]$. We use that for matrices $\bfA,\bfB$ of appropriate dimensions, $\|\bfA \bfB\|_2 \le \|\bfA\|_2 \|\bfB\|_2$ and that $\|\bfA \otimes \bfB\|_2 = \|\bfA\|_2 \|\bfB\|_2$. Note also that $\|\SelectKron_j\|_2 = 1$  for $j=1,\dots,\ell$. Recall that $\tbfx_0$,$\bfu_0,\dots,\bfu_{k-1}$, $\tbfA_1,\dots,\tbfA_{\ell}$, $\tbfB$, $\SelectKron_1,\dots,\SelectKron_{\ell}$ are deterministic quantities with finite norm.
        We thus have, for some finite constant $\bar{C} (j_1,\dots,j_{\ell + 1})  >  0$, the bound
        \begin{align} \label{eq:ZetaBound}
               & \Exp \left[ \| \zeta_{j_1,\dots,j_{\ell+1}}(\bfE_{\hbfA_1},\dots,\bfE_{\hbfA_{\ell}},\bfE_{\hbfB}) \|_2 \right] \\
            & \le \bar{C} (j_1,\dots,j_{\ell + 1})  \Exp [\|\bfE_{\hbfA_1}\|_2^{j_{1}} \cdots \|\bfE_{\hbfA_{\ell}}\|_2^{j_{\ell}} \|\bfE_{\hbfB}\|_2^{j_{\ell+1}}] \notag \\
            & \le \bar{C} (j_1,\dots,j_{\ell + 1})  \sigma^{j_1+\dots+j_{\ell+1}} \|\Sigma_{\hbfA_1}^{1/2}\|_2^{j_1} \cdots \|\Sigma_{\hbfA_{\ell}}^{1/2}\|_2^{j_{\ell}} \|\Sigma_{\hbfB}^{1/2}\|_2^{j_{\ell+1}} \Exp [\| \bfG_{\hbfA_1}\|_2^{j_{1}} \cdots \| \bfG_{\hbfA_{\ell}} \|_2^{j_{\ell}} \| \bfG_{\hbfB}\|_2^{j_{\ell + 1}}]\notag \\
            & \le \bar{C} (j_1,\dots,j_{\ell + 1}) 
            \left( \frac{\sigma}{\singValMin(\bfD)} \right)^{j_{1}+\dots+j_{\ell+1}} \Exp [\| \bfG_{\hbfA_1}\|_2^{j_{1}} \cdots \| \bfG_{\hbfA_{\ell}} \|_2^{j_{\ell}} \| \bfG_{\hbfB}\|_2^{j_{\ell + 1}}] \notag
        \end{align}
    where we utilized \eqref{eq:NormSigmaB}, \eqref{eq:NormSigmaA}. 
    
    Recursively applying the Cauchy-Schwarz inequality and invoking concentration inequalities on $\|\bfG_{\hbfB}\|_2, \|\bfG_{\hbfA_s}\|_2, s = 1,\dots,\ell$  shows that $\Exp [\| \bfG_{\hbfA_1}\|_2^{j_1} \cdots \| \bfG_{\hbfA_{\ell}} \|_2^{j_{\ell}} \| \bfG_{\hbfB}\|_2^{j_{\ell + 1}}]$  is finite. To illustrate this, consider 
    \begin{multline} \label{eq:CauchySchwarzBnd}
        \biggl | \Exp [\| \bfG_{\hbfA_1}\|_2^{j_{1}} \cdots \| \bfG_{\hbfA_{\ell}} \|_2^{j_{\ell}} \| \bfG_{\hbfB}\|_2^{j_{\ell + 1}}] \biggr | \\ \le \sqrt{\Exp [\| \bfG_{\hbfA_1}\|_2^{2j_{1}}] \Exp[ \| \bfG_{\hbfA_{2}} \|_2^{2j_{2}} \cdots \| \bfG_{\hbfA_{\ell}} \|_2^{2j_{\ell}} \| \bfG_{\hbfB}\|_2^{2j_{\ell + 1}}]}.
    \end{multline}    
    By invoking Lemma~\ref{lemma:GaussNormPowerBound}, we can obtain a bound for $\Exp[\|\bfG_{\hbfA_1}\|_2^{2j_1}]$. The Cauchy-Schwarz inequality is then applied to $\Exp[ \| \bfG_{\hbfA_{2}} \|_2^{2j_2} \cdots \| \bfG_{\hbfA_{\ell}} \|_2^{2j_{\ell}} \| \bfG_{\hbfB}\|_2^{2j_{\ell + 1}}]$ after which concentration inequalities are invoked to bound $\Exp[\|\bfG_{\hbfA_2}\|_2^{4j_2}]$, which is repeated $\ell$ times until the expected value of products is decomposed into a product of expected values.
    The proposition then follows from \eqref{eq:ZetaBound} by summing over the indices for which $j_{1}+\dots+j_{\ell+1} = l$ with $l = 2, \dots, Q_k$.
    \end{proof}
    
We now derive a bound for the MSE of the predicted states, which shows that for polynomially nonlinear systems, the MSE in the asymptotic regime $\sigma/\singValMin(\bfD) \rightarrow 0$ is dominated by  $(\sigma/\singValMin(\bfD))^2$.
\begin{proposition} \label{prop:nonlinearpolyMSE}
         Let $\hbfx_0 = \tbfx_0$. Suppose that the conditions of Proposition~\ref{prop:LSOperators} hold. If the high-dimensional system \eqref{eq:NoisyHighDimSys} from which data are sampled and the learned low-dimensional model are polynomially nonlinear,  then
    $$\Exp[\|\hbfx_k - \tbfx_k\|^2_2] \le \sum_{l=2}^{2Q_k} \hat{C}_l \left(\frac{\sigma}{\singValMin(\bfD)}\right)^l\,,\qquad 1 \leq k \in \mathbb{N}, $$
    for some constants $0 < \hat{C}_l < \infty, l =2,\dots,2Q_k$, which are not functions of $\sigma$ and $\bfD$.     
     \end{proposition}

\begin{proof}
    From \eqref{eq:NonlinxHatTimeStep} and \eqref{eq:NonlinxTildeTimeStep}, we obtain
    \begin{align*}
        \hbfx_k - \tbfx_k = \sum_{l=1}^{Q_k} \sum_{\substack{j_1,\dots,j_{\ell +1} \in \mathbb{N} \\j_{1}+\dots+j_{\ell + 1} = l}} \zeta_{j_1,\dots,j_{\ell + 1}}(\bfE_{\hbfA_1},\dots,\bfE_{\hbfA_{\ell}},\bfE_{\hbfB}).
    \end{align*}
    Thus, for some finite constant $\hat{C}(j_1,\dots,j_{\ell + 1}) > 0$,
    \begin{align*}
        & \|\hbfx_k - \tbfx_k\|_2 \\
        &  \le \sum_{l=1}^{Q_k} \sum_{\substack{j_1,\dots,j_{\ell +1} \in \mathbb{N} \\j_{1}+\dots+j_{\ell + 1} = l}} \| \zeta_{j_1,\dots,j_{\ell+1}}(\bfE_{\hbfA_1},\dots,\bfE_{\hbfA_{\ell}},\bfE_{\hbfB}) \|_2 \\
        &  \le  \sum_{l=1}^{Q_k} \sum_{\substack{j_1,\dots,j_{\ell +1} \in \mathbb{N} \\j_{1}+\dots+j_{\ell + 1} = l}} \hat{C} (j_1,\dots,j_{\ell + 1})  \|\bfE_{\hbfA_1}\|_2^{j_{1}} \cdots \|\bfE_{\hbfA_{\ell}}\|_2^{j_{\ell}} \|\bfE_{\hbfB}\|_2^{j_{\ell+1}} \\
        & \le \sum_{l=1}^{Q_k} \sum_{\substack{j_1,\dots,j_{\ell +1} \in \mathbb{N} \\j_{1}+\dots+j_{\ell + 1} = l}} \hat{C} (j_1,\dots,j_{\ell + 1}) \left( \frac{\sigma}{\singValMin(\bfD)} \right)^{j_{1}+\dots+j_{\ell+1}} \| \bfG_{\hbfA_1}\|_2^{j_{1}} \cdots \| \bfG_{\hbfA_{\ell}} \|_2^{j_{\ell}} \| \bfG_{\hbfB}\|_2^{j_{\ell + 1}}
    \end{align*}
    following calculations in \eqref{eq:ZetaBound} where $\bfG_{\hbfA_1},\dots,\bfG_{\hbfA_{\ell}},\bfG_{\hbfB}$ are the same random matrices defined in the proof of Proposition~\ref{prop:nonlinearpoly}. Note that in the above inequality, the powers of the noise-to-signal ratio range from $j_1+\dots+j_{\ell+1} = 1$ to $j_1+\dots+j_{\ell+1} = Q_k$, whereas in the proof of Proposition~\ref{prop:nonlinearpoly} the inequality \eqref{eq:ProofInEqBias} starts at $j_1+\dots+j_{\ell+1} = 2$. The conclusion now follows by squaring both sides of the inequality, applying expectation, and performing calculations similar to \eqref{eq:CauchySchwarzBnd} to show that the resulting constants are finite.
\end{proof}

\section{Active operator inference for selecting training data} \label{sec:ActiveOpInf}
This section proposes active operator inference, which selects from a dictionary at which initial condition and inputs to sample the high-dimensional system for generating data with low noise-to-signal ratios. The proposed active operator inference is motivated by the bounds derived in Section~\ref{subsec:ExpError}, which show that the noise-to-signal ratio $\sigma/\singValMin(\bfD)$ controls the MSE of the learned operators as well as the bias and the MSE of the state dynamics. 

Section~\ref{subsec:dictionary} formalizes the dictionary whose elements are candidates for sampling the high-dimensional system at. The proposed selection of elements of the dictionary is described in Section~\ref{subsec:ODEIM} and builds on ideas from selecting points \cite{PeherstorferODEIM2020} in empirical interpolation \cite{BARRAULT2004667,doi:10.1137/090766498}. The computational procedure for active operator inference is presented in Section~\ref{subsec:Aopinf}, which summarizes the proposed workflow for learning low-dimensional models from noisy data.

\subsection{Dictionary of candidate states and inputs} \label{subsec:dictionary}

Consider a dictionary $\Dcal \in \R^{L \times M}$ of candidate states and inputs given by 
\[
\Dcal = [\bbfX_L^T, (\bbfX_L^2)^T, \dots, (\bbfX_L^{\ell})^T, \bfU_L^T]\,,
\]
where $\bbfX_L \in \R^{n \times L}, \bfU_L \in \R^{p \times L}$ are defined identically as $\bbfX, \bfU$ in Section~\ref{subsec:opinfNoisy} but with $L$ states and inputs, i.e. $\bbfX_L = [\bbfx_1,\dots,\bbfx_L]$ and $\bfU_L = [\bfu_1,\dots,\bfu_L]$. Let $\bfP_K \in \{0, 1\}^{L \times K}$ be a selection operator that selects $K \leq L$ rows of $\Dcal$ via $\bfP_K^T\Dcal = \bfD$ so that $\bfD$ can serve as a data matrix in the sense of \eqref{eq:LeastSquaresNoisy}. Observe that using all rows of $\Dcal$ to construct the data matrix $\bfD$ in the least squares problem \eqref{eq:LeastSquaresNoisy} is computationally expensive as the high-dimensional system has to be queried for each of the $L$ initial conditions and inputs. 

\subsection{A design of experiments approach via oversampled empirical interpolation}
\label{subsec:ODEIM}
Propositions~\ref{proposition:linearExpErr}, \ref{prop:nonlinearpoly}, and \ref{prop:nonlinearpolyMSE} demonstrate that a low noise-to-signal ratio $\sigma / \singValMin(\bfD)$ is desirable. Since the standard deviation $\sigma$ is fixed, we propose a design of experiments strategy that forms a data matrix $\bfD$ by selecting rows of $\Dcal$ so that $\singValMin(\bfD)$ is large.
To find a selection of rows, we follow the procedure proposed in \cite{PeherstorferODEIM2020} that pursues an equivalent objective for selecting points for empirical interpolation \cite{BARRAULT2004667,doi:10.1137/090766498,Drma2016}; see \cite{4668528,8361090,9152984,doi:10.1137/16M1057668} for other design of experiment approaches based on similar linear-algebra concepts. Set $K \ge M$. The method introduced in \cite{PeherstorferODEIM2020} constructs a selection matrix $\bfP_K$ which selects $K$ rows of $\Dcal$ with the objective of maximizing $\singValMin(\bfP_K^T \Dcal)$. First, the selection matrix $\bfP_M \in \R^{M \times M}$ is initialized with the approach introduced in \cite{Drma2016}. Then, new rows of $\Dcal$ are selected in a greedy fashion. To describe the greedy update, suppose we have the selection matrix $\bfP_m$, which selects $m$ rows of $\Dcal$, with $ M \le m < K$. Let the SVD of $\bfP_m^T \Dcal$ be  $\bfP_m^T \Dcal = \boldsymbol{\Phi}_m \boldsymbol{\Sigma}_m \boldsymbol{\Psi}_m^T$ where $\boldsymbol{\Phi}_m \in \R^{m \times M}$ is the matrix of left-singular vectors, $\boldsymbol{\Sigma}_m \in \R^{ M \times  M}$ is the diagonal matrix of singular values $s_1^{(m)},\dots,s_{M}^{(m)}$ in descending order, and $\boldsymbol{\Psi}_m \in \R^{M \times  M}$ is the matrix of right-singular vectors. Define the gap $g = (s_{M-1}^{(m)})^2 - (s_{M}^{(m)})^2$ and set $\rbfd_+ = \boldsymbol{\Psi}_m^T \bfd_+^T$, where $\bfd_+ \in \R^{1 \times M}$ is a candidate row of $\Dcal$ that has not been selected by $\bfP_m$. Further, let $\bfe \in \R^{M}$ be the canonical basis vector with all entries 0 except for the last component that is set to 1. It is shown in \cite{doi:10.1137/070682745}, see also the discussion in \cite{PeherstorferODEIM2020}, that
    \begin{align} \label{eq:singValUpdate}
        \singValMin(\bfP_{m+1}^T \Dcal)^2 - \singValMin(\bfP_{m}^T \Dcal)^2 \ge \frac{1}{2} \left( g+ \|\rbfd_+\|_2^2 - \sqrt{(g+\|\rbfd_+\|_2^2)^2 - 4g(\bfe^T \rbfd_{+})^2} \right)
    \end{align}
    which suggests that the new row $\bfd_{+}$ should be selected that maximizes the lower bound \eqref{eq:singValUpdate}. 
    This greedy step is then repeated until the desired number of rows $K$ is reached.
    
 Based on the just described greedy scheme, we use a modified greedy update rule, which was proposed in an earlier preprint version of \cite{PeherstorferODEIM2020}:  we choose the new row $\bfd_+$ that maximizes
    \begin{align} \label{eq:oldODEIMCriterion}
        (\bfe^T \rbfd_+)^2\,,
    \end{align}
    which is obtained by simplifying the lower bound \eqref{eq:singValUpdate}.

		Choosing the new row $\bfd_+$ by maximizing the lower bound \eqref{eq:singValUpdate} was tested in \cite{PeherstorferODEIM2020} for the case where the columns of $\Dcal$ are orthonormal. Since this condition does not necessarily hold in our setting, the lower bound in \eqref{eq:singValUpdate} for the greedy update can lead to cancellation errors especially when $4g(\bfe^T \rbfd_{+})^2$ is small, which  has been first observed in \cite{doi:10.2514/6.2021-1371}. 
		
\subsection{Active operator inference} \label{subsec:Aopinf}

The proposed active operator inference approach to learn low-dimensional models from noisy data is summarized in  Algorithm \ref{alg:ODEIMOversampling}. The inputs of the algorithm are the dictionary $\Dcal$ and the number of times $K$ to query the high-dimensional system, i.e., the number of rows of the data matrix $\bfD$. Line 2 of Algorithm~\ref{alg:ODEIMOversampling} initializes the sampling matrix via QDEIM \cite{Drma2016} by computing the QR decomposition of $\Dcal^T$ with pivoting. For $m \in \{M, M+1, \dots, K-1\}$, the SVD of $\bfP_m^T\Dcal$ is obtained in line 4 and the candidate row $\bfd_{+}$ of $\Dcal$ that maximizes \eqref{eq:oldODEIMCriterion} is selected in in lines 5--6 to update $\bfP_m$ to $\bfP_{m+1}$. In lines 7--9, re-projection \eqref{eq:NoisyReProj} is performed using the projected states and the inputs in the data matrix $\bfD = \bfP_K^T \Dcal$ to obtain the re-projected trajectory $\bbfZ$. The least-squares problem \eqref{eq:LeastSquaresNoisy} is then solved to learn the low-dimensional operators.

\begin{algorithm}[t] 
\caption{Active operator inference based on QDEIM \cite{Drma2016} and oversampling \cite{PeherstorferODEIM2020}}
\begin{algorithmic}[1]

\Procedure{AOpInf}{$\Dcal,K$} 
    \State Initialize $\bfP_M$ with QDEIM \cite{Drma2016}
    \For{$m = M,\dots, K-1$} \Comment{Follow \cite{PeherstorferODEIM2020} with criterion \eqref{eq:oldODEIMCriterion}}
    	\State Compute the SVD of $\bfP_m^T \Dcal = \boldsymbol{\Phi}_m \boldsymbol{\Sigma}_m \boldsymbol{\Psi}_m^T$ 
    	\State Find the row $\bfd_+$ of $\Dcal$ not in $\bfP_m^T \Dcal$ such that $\rbfd_+ = \boldsymbol{\Psi}_m^T \bfd_+^T$ maximizes \eqref{eq:oldODEIMCriterion}
        \State Update $\bfP_m$ to $\bfP_{m+1}$
     \EndFor
    \State Construct $\bfD$ via $\bfP_K^T \Dcal = \bfD$
    \State Perform re-projection as in \eqref{eq:NoisyReProj} to generate $\bbfZ$
    \State Solve the least-squares problem \eqref{eq:LeastSquaresNoisy} to obtain $\hbfA_1,\dots,\hbfA_{\ell},\hbfB$
\EndProcedure
\Return $\hbfA_1,\dots,\hbfA_{\ell},\hbfB$ 
\end{algorithmic}\label{alg:ODEIMOversampling}
\end{algorithm}

\section{Numerical experiments} \label{sec:NumExp}

We now numerically demonstrate that the proposed active operator inference leads to predicted states with orders of magnitude lower biases and MSEs than an uninformed equidistant-in-time selection of data samples. Additionally, we demonstrate that the bias and MSE of predicted states decay with the noise-to-signal ratio in agreement with the analysis developed in Section~\ref{subsec:ExpError}. Numerical results for a linear state dynamics are shown in Section~\ref{subsec:SteelRail} and for quadratic dynamics in Section~\ref{subsec:DiffLotkaVolterra}. In all experiments within one example, we use the same basis matrix $\bfV$, which ensures consistent comparisons among different noise-to-signal ratios. 

\subsection{Heat transfer problem for cooling of steel profiles} \label{subsec:SteelRail}
The model and problem setup are described in Section~\ref{subsec:SteelModel} and the numerical results are presented in Section~\ref{subsec:SteelResults}.

\subsubsection{Model of cooling steel profiles} \label{subsec:SteelModel}

 We describe a mathematical model for the cooling process of steel rail profiles in a rolling mill following \cite{BenS05b}. Set $\Omega \subset \R^2$ as the spatial domain and denote by $x(\bfeta,t)$ the temperature at the spatial point $\bfeta \in \Omega$ and time $t > 0$. The heat transfer model is
    \begin{align} \label{eq:HeatEqn}
        \frac{\partial x(\bfeta,t)}{\partial t} & = \frac{\lambda}{c \rho} \Delta x(\bfeta,t), \quad (\bfeta,t) \in \Omega \times [0,T], \\
        \nabla x(\bfeta,t) \cdot \textbf{n} & = \begin{cases}
         \frac{\kappa}{\lambda}(u_j(t) - x(\bfeta,t)) & \quad \text{for} \quad \bfeta \in \Gamma_j, j = 1,\dots,7,\\
            0 & \quad \text{for} \quad \bfeta \in \Gamma_0,
        \end{cases} 
        \notag \\
        x(\bfeta,0) & =  500, \notag
    \end{align}
    where $\lambda$ is the heat conductivity, $c$ the specific heat capacity, $\rho$ the profile density, $\kappa$ the heat transfer coefficient, $\Gamma_j, j = 0,\dots,7$ are segments of the domain boundary $\partial \Omega$ such that $\partial \Omega = \cup_{j=0}^7 \Gamma_j$ and $u_j(t), j =1,\dots,7$ is the external temperature applied to each boundary segment. The domain is visualized in Figure~\ref{fig:SteelProfileDomain}. The values of the constants are chosen as $\lambda = 26.4, c = 7620, \rho = 654, \kappa = 69.696$. 
    
Equation \eqref{eq:HeatEqn} is spatially discretized using the finite element method with linear triangular elements and temporally discretized with implicit Euler with step size $\delta t = 0.01$ to yield the high-dimensional system \eqref{eq:NoisyHighDimSys} with right-hand side function~\eqref{eq:PolySys} with $\ell = 1$, where $\bfx_k \in \R^N, N = 1357$ and $\bfu_k \in \R^p, p = 7$. We utilized the Python\footnote{  \href{https://gitlab.mpi-magdeburg.mpg.de/models/fenicsrail/-/tree/master/}{https://gitlab.mpi-magdeburg.mpg.de/models/fenicsrail/-/tree/master/}} code based on the FEniCS Project to generate the computational mesh and the system matrices; see also \cite{BenS05b}.
    
The basis matrix $\bfV$ is computed from snapshots $\bfx^{\text{basis}}_k, k = 0,\dots,L, L = 10000$, of the high-dimensional system  driven by the input $\bfu_k^{\text{basis}}$ whose $i$-th component, $i=1,\dots,7$ is given by $500(1 - \tanh(k \delta t/i^2)) + 250\gamma_{i,k}$. Here,  $\gamma_{i,0} = 0$ while $\gamma_{i,k}$ for $k>0$ is a realization of a uniform random variable on $[0,1]$. The projected states $\bfV^T \bfx_k^{\text{basis}}$ and the input $\bfu_k^{\text{basis}}$ for $k = 0,\dots,L-1$ constitute the 10000 rows of the dictionary $\Dcal \in \R^{L \times (n+p)}$ from which we select the rows of the data matrix $\bfD$ for operator inference. In the simulations below, we consider 5 equally spaced values in the logarithm scale for the standard deviation $\sigma$ of the noise between $1\times 10^{-3}$ and $1 \times 10^{-1}$. The test input $\bfu_k^{\text{test}}$ at time step $k \in \mathbb{N}$ has components given by $500(1 - \tanh(k \delta t/i^2))$ for $i = 1,\dots,7$.
   
\begin{figure}
  \centering
    \includegraphics[width=0.7\textwidth,trim={10cm 10cm 10cm 10cm},clip]{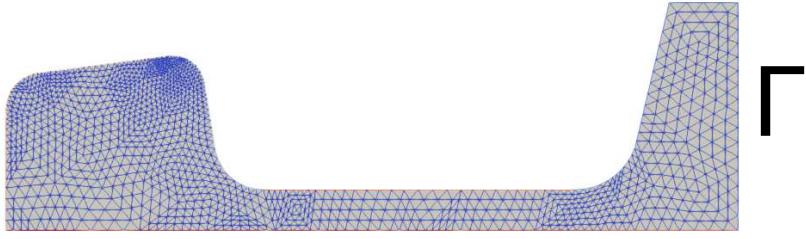}
    \caption{Steel profile domain $\Omega$ for the heat transfer problem in Section~\ref{subsec:SteelRail}.}
    \label{fig:SteelProfileDomain}
\end{figure}

\subsubsection{Results} \label{subsec:SteelResults}

We learn a low-dimensional model of dimension $n=7$ and $n=10$ from noisy data. For $n=7$, active operator inference is applied to select 15 rows from $\Dcal$, which leads to a data matrix $\bfD$ with $\singValMin(\bfD) = 1.661$. For $n=10$, 25 rows are selected resulting in $\singValMin(\bfD) = 0.8713$. Denote by $\hbfx_k^{\text{test}}$ the predicted state at time step $k$ of the low-dimensional model with inferred operators $\hbfA,\hbfB$ corresponding to the test input $\bfu_k^{\text{test}}$. Likewise, let $\tbfx_k^{\text{test}}$ be the low-dimensional state from intrusive model reduction for the same input. Recall that $\tbfx_k^{\text{test}}$ is deterministic while $\hbfx_k^{\text{test}}$ is a random vector.

\begin{figure}
  \begin{subfigure}[b]{0.45\textwidth}
    \begin{center}
{{\Large\resizebox{1.15\columnwidth}{!}{\input{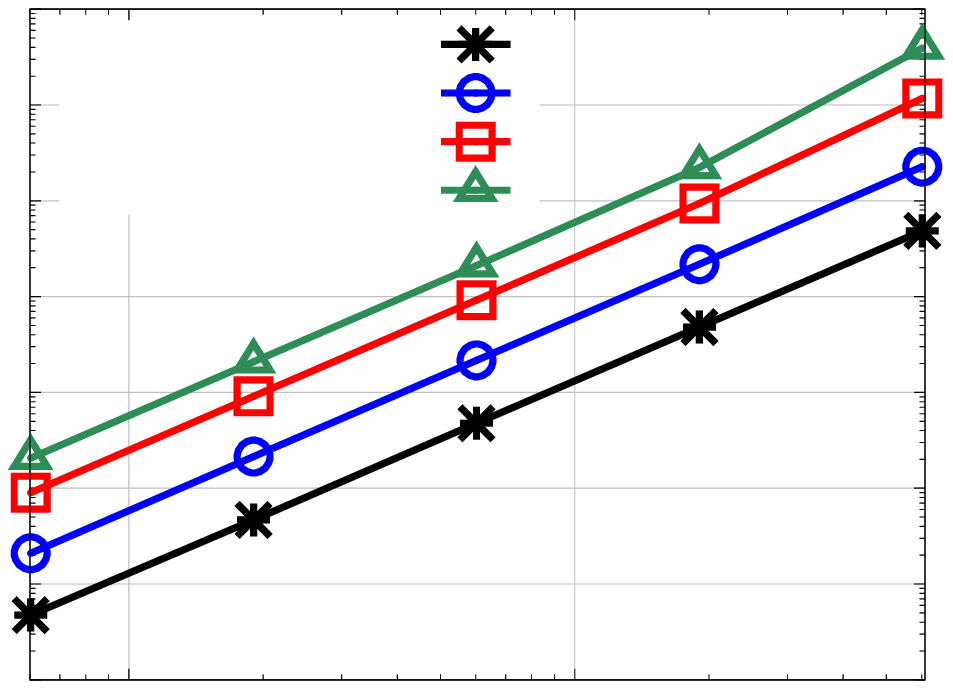}}}}
\end{center}
    \caption{dimension $n = 7$}
  \end{subfigure}
  \quad \quad \quad
  \begin{subfigure}[b]{0.45\textwidth}
    \begin{center}
{{\Large\resizebox{1.15\columnwidth}{!}{\input{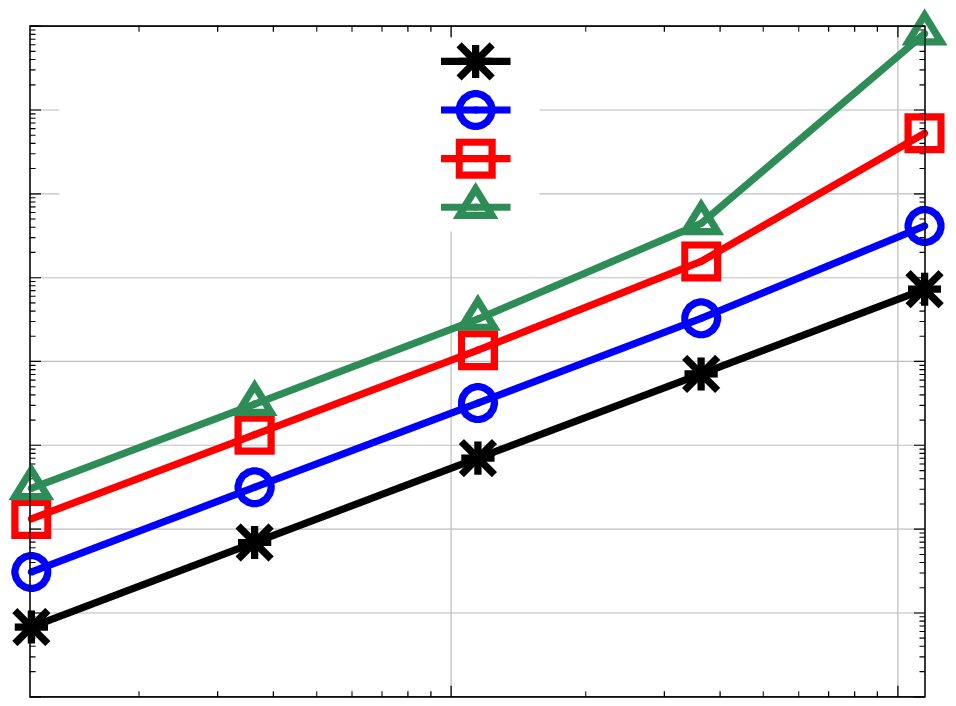}}}}
\end{center}
    \caption{dimension $n=10$}
  \end{subfigure}
  \caption{Cooling of steel profiles (Section~\ref{subsec:SteelRail}). The estimated bias decays by 2 orders of magnitude per 1 order of magnitude decrease in the noise-to-signal ratio in the asymptotic regime. The results are in agreement with Proposition~\ref{proposition:linearExpErr}.} 
  \label{fig:SteelProfile_decay}
\end{figure}
	
Figure~\ref{fig:SteelProfile_decay} shows a Monte Carlo estimate of the bias $\|\Exp[\hbfx_k^{\text{test}} - \tbfx_k^{\text{test}}]\|_2$ as a function of the noise-to-signal ratio $\sigma/\singValMin(\bfD)$ for various time steps $k$. A Monte Carlo estimate of the MSE  $\Exp[\|\hbfx_k^{\text{test}} - \tbfx_k^{\text{test}}\|_2^2]$ is shown in the left panel of Figure~\ref{fig:SteelProfile_MSE}. We use $7.5 \times 10^7$ samples to approximate the expected value with Monte Carlo. The plots illustrate that in the asymptotic regime, when $\sigma/\singValMin(\bfD) \rightarrow 0$, an order decrease in the noise-to-signal ratio leads to a decrease of two orders of magnitude in the approximation of the bias and the MSE, which agrees with Propositions \ref{proposition:linearExpErr} and \ref{prop:nonlinearpolyMSE}. Notice that for $n=10$, the behavior of the bias and the MSE for the largest noise value $\sigma$ is already dominated by constants, rather than the noise-to-signal ratio, which explains the quicker error increase.

\begin{figure}
  \begin{subfigure}[b]{0.45\textwidth}
    \begin{center}
{{\Large\resizebox{1.15\columnwidth}{!}{\input{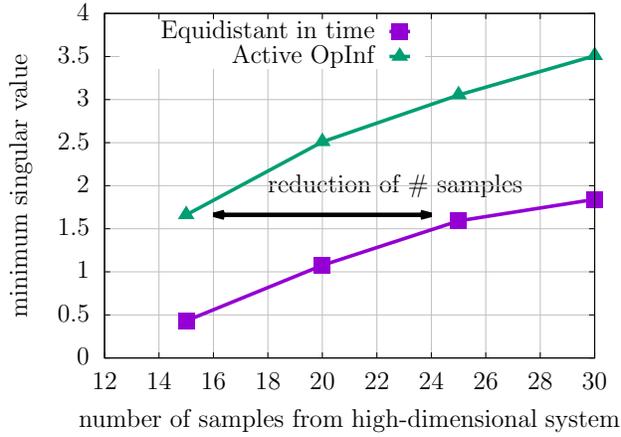}}}}
\end{center}
    \caption{dimension $n=7$}
  \end{subfigure}
  \quad \quad \quad
  \begin{subfigure}[b]{0.45\textwidth}
    \begin{center}
{{\Large\resizebox{1.15\columnwidth}{!}{\input{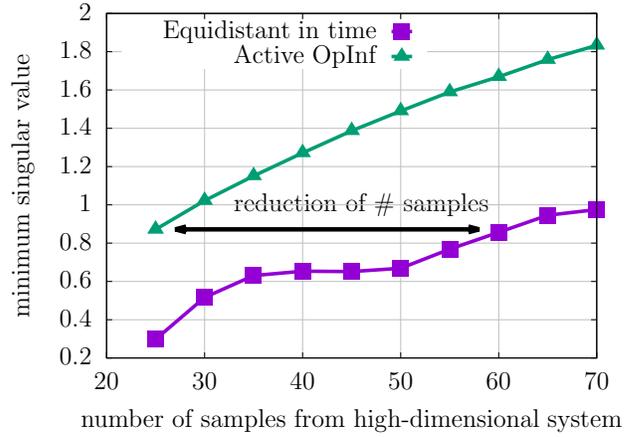}}}}
\end{center}
    \caption{dimension $n=10$}
  \end{subfigure}
  \caption{Cooling of steel profiles (Section~\ref{subsec:SteelRail}). To achieve the same noise-to-signal ratio, active operator inference requires almost 3 times fewer queries to the high-dimensional system than a traditional selection of equidistant-in-time samples.}
  \label{fig:SteelProfile_DOE_singval}
\end{figure}

\begin{figure}
  \begin{subfigure}[b]{0.45\textwidth}
    \begin{center}
{{\Large\resizebox{1.15\columnwidth}{!}{\input{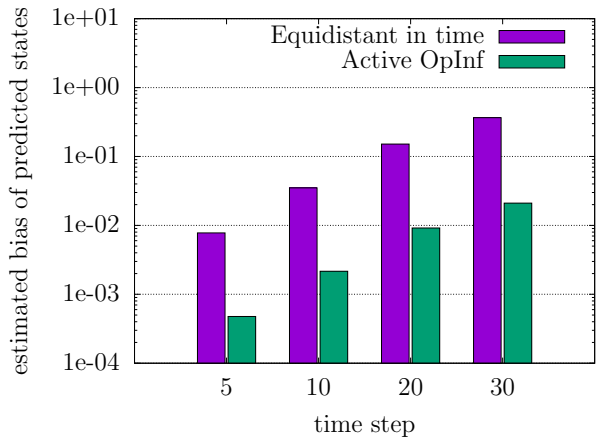}}}}
\end{center}
    \caption{dimension $n=7$}
  \end{subfigure}
  \quad \quad \quad
  \begin{subfigure}[b]{0.45\textwidth}
    \begin{center}
{{\Large\resizebox{1.15\columnwidth}{!}{\input{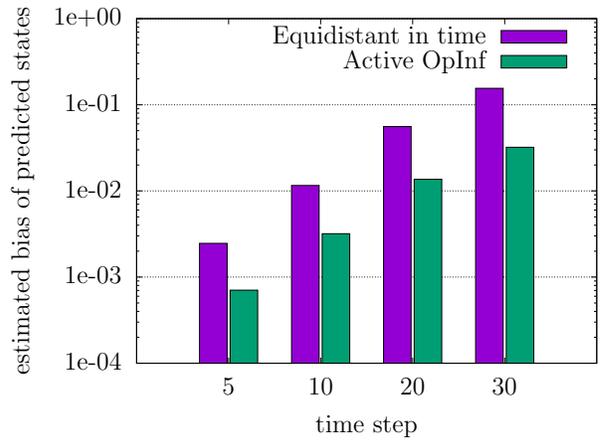}}}}
\end{center}
    \caption{dimension $n=10$}
  \end{subfigure}
  \caption{Cooling of steel profiles (Section~\ref{subsec:SteelRail}). Active operator inference yields predictions which have a lower bias compared to the predictions delivered by sampling equidistantly in time in the dictionary. The reduction in the estimated bias achieved by active operator inference is up to 1.5 orders in magnitude.}
  \label{fig:SteelProfile_DOE}
\end{figure}

We now compare active operator inference, which carefully selects rows of the data matrix $\bfD$ to keep the noise-to-signal ratio low, with a traditional sample selection that queries the high-dimensional system equidistantly in time, i.e., picks columns corresponding to equidistant times from the dictionary $\Dcal$. Figure~\ref{fig:SteelProfile_DOE_singval} compares the minimum singular value of the data matrix for both approaches over the number of queries to the high-dimensional system. Equidistant sampling requires up to 3 times as many queries to the high-dimensional system to achieve the same noise-to-signal ratio as active operator inference in our experiment. The estimate of the bias $\|\Exp[\hbfx_k^{\text{test}} - \tbfx_k^{\text{test}}]\|_2$ for $\sigma = 1\times 10^{-2}$ for the equidistant and active operator inference approach is presented in Figure~\ref{fig:SteelProfile_DOE}. The same comparison for the MSE $\Exp[\|\hbfx_k^{\text{test}} - \tbfx_k^{\text{test}}\|_2^2]$ is shown in the right panel of Figure \ref{fig:SteelProfile_MSE} for $n=10$. The results show that active operator inference yields a reduction in the estimated bias and the MSE of up to 1.5 and 0.5 orders in magnitude, respectively. Lastly, we consider the MSE of the predicted state further in time. The bottom panel of Figure \ref{fig:SteelProfile_MSE} plots the estimated MSE of the predicted state at 10000 time steps for $n=10$ using 10 Monte Carlo samples only. An order decay in the noise standard deviation leads to 2 orders decay in the estimated MSE. For fixed $\sigma$, the model learned through active operator inference achieves a smaller MSE.
	
In Figure~\ref{fig:SteelProfile_DOE_States} we visualize the 15 high-dimensional states corresponding to the rows of $\Dcal$ selected according to the design of experiments schemes we compare for $n=7$. The respective inputs are not shown. By examining the segments of the steel profile boundary with Robin condition, the equidistant scheme tends to select more states with lower temperature at the boundary, many of which correspond to later time steps. In contrast, active operator inference selects more states at the beginning of the cooling process.

\begin{figure}
  \begin{subfigure}[b]{0.45\textwidth}
    \begin{center}
{{\Large\resizebox{1.15\columnwidth}{!}{\input{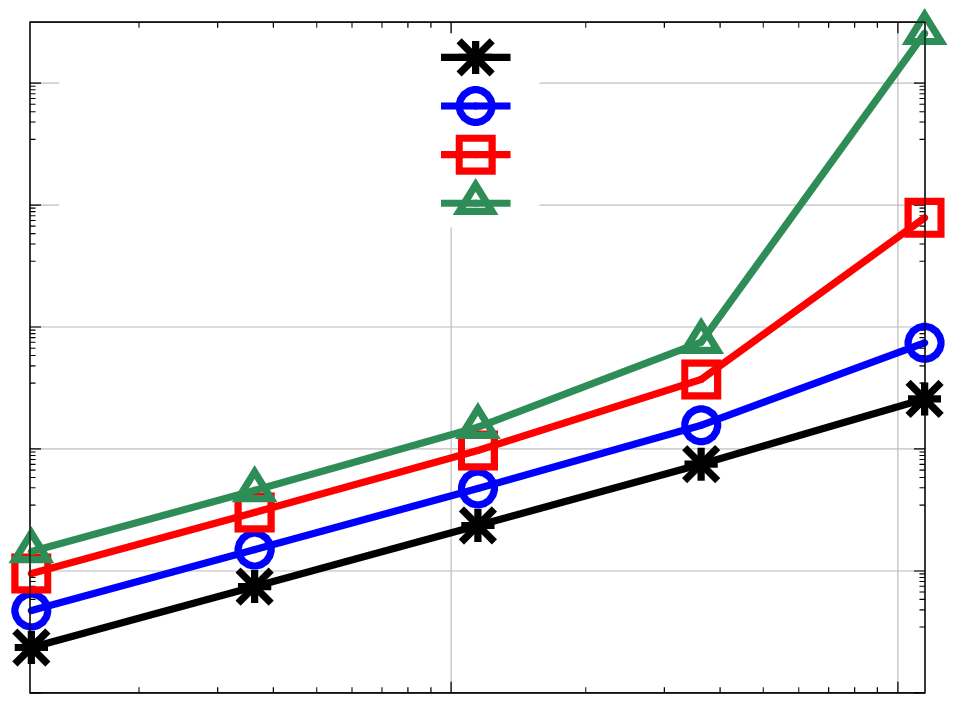}}}}
\end{center}
    \caption{decay of MSE, $n=10$}
  \end{subfigure}
  \quad \quad \quad
  \begin{subfigure}[b]{0.45\textwidth}
    \begin{center}
{{\Large\resizebox{1.15\columnwidth}{!}{\input{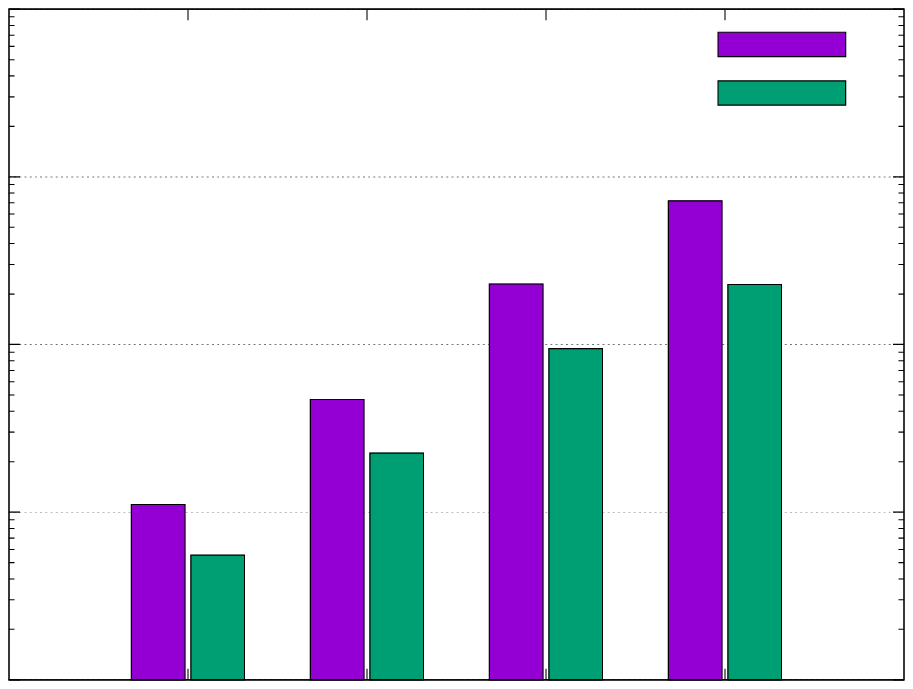}}}}
\end{center}
    \caption{equidistant vs Active OpInf, $n=10$}
  \end{subfigure}
 \begin{center}
  \begin{subfigure}[b]{0.45\textwidth}
    \begin{center}
{{\Large\resizebox{1.15\columnwidth}{!}{\input{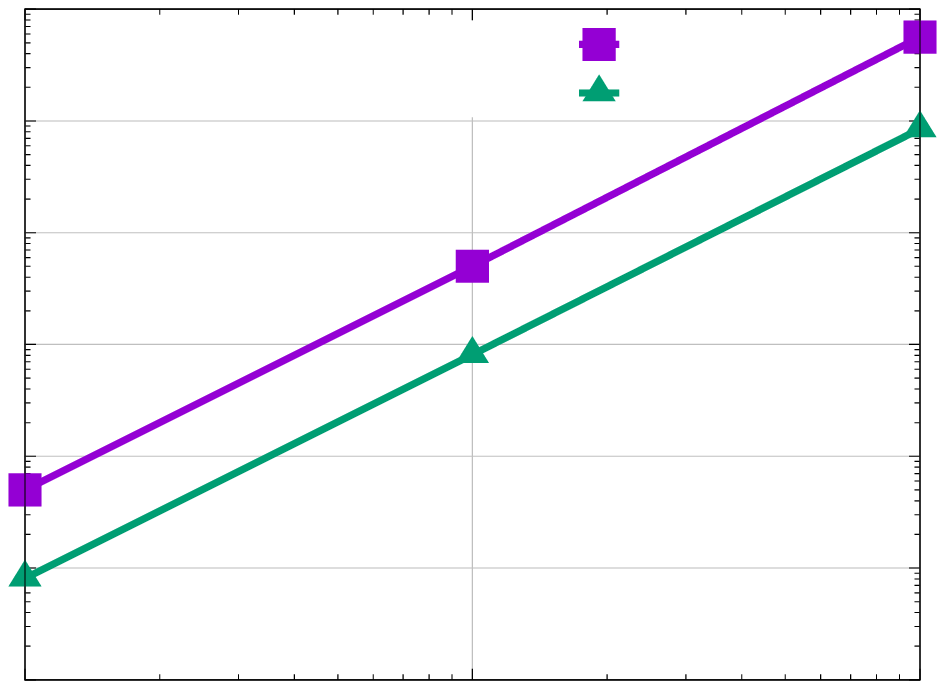}}}}
\end{center}
    \caption{time step 10000}
  \end{subfigure}
 \end{center}
  \caption{Cooling of steel profiles (Section~\ref{subsec:SteelRail}). In the asymptotic regime, the estimated MSE decays by 2 orders of magnitude per 1 order of magnitude decrease in the noise-to-signal ratio, which agrees with Proposition~\ref{prop:nonlinearpolyMSE}. The predictions obtained with active operator inference have a lower MSE than those obtained from equidistant-in-time samples.}
  \label{fig:SteelProfile_MSE}
\end{figure}

\begin{figure}
  \begin{subfigure}[b]{0.22\textwidth}
    \begin{center}
\includegraphics[width=0.45\textwidth,trim={6cm 1cm 6cm 1cm},clip]{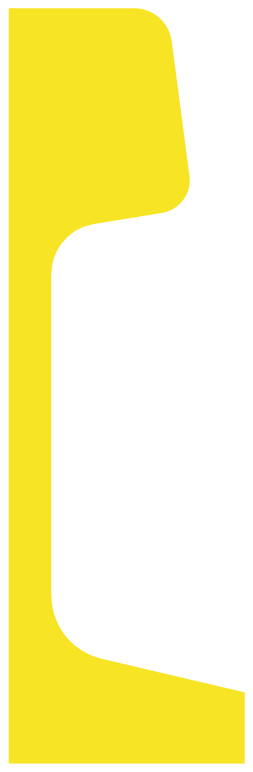}
\includegraphics[width=0.45\textwidth,trim={6cm 1cm 6cm 1cm},clip]{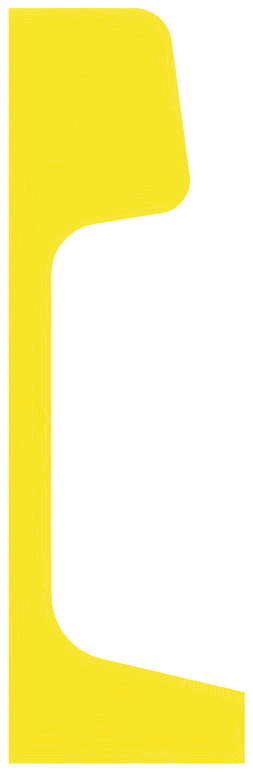}
\end{center}
    \caption{State 1}
  \end{subfigure}
  \hspace{0.7em}
  \begin{subfigure}[b]{0.22\textwidth}
    \begin{center}
\includegraphics[width=0.45\textwidth,trim={6cm 1cm 6cm 1cm},clip]{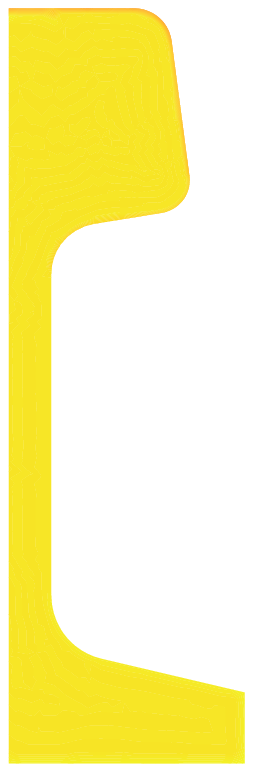}
\includegraphics[width=0.45\textwidth,trim={6cm 1cm 6cm 1cm},clip]{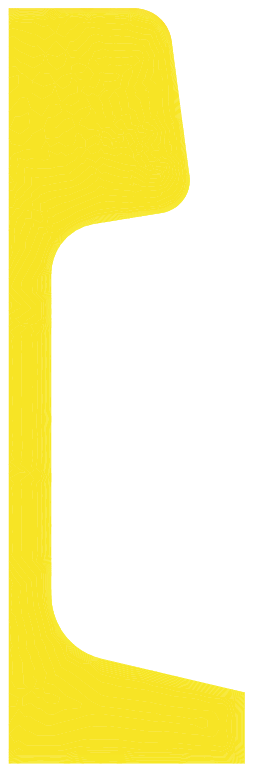}
\end{center}
    \caption{State 2}
  \end{subfigure}
  \hspace{0.7em}
  \begin{subfigure}[b]{0.22\textwidth}
    \begin{center}
\includegraphics[width=0.45\textwidth,trim={6cm 1cm 6cm 1cm},clip]{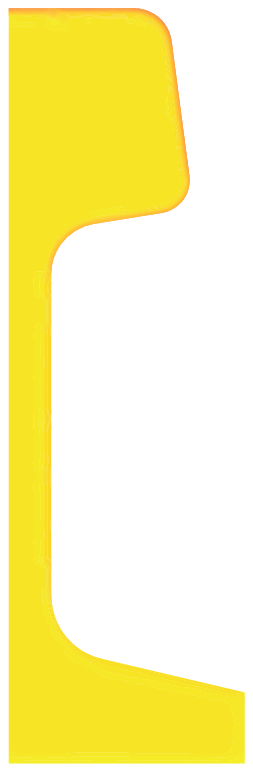}
\includegraphics[width=0.45\textwidth,trim={6cm 1cm 6cm 1cm},clip]{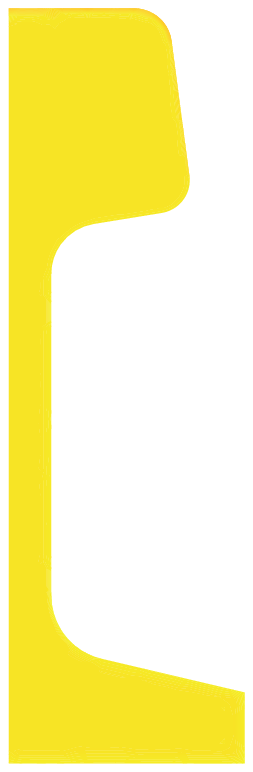}
\end{center}
    \caption{State 3}
  \end{subfigure}
  \hspace{0.7em}
  \begin{subfigure}[b]{0.22\textwidth}
    \begin{center}
\includegraphics[width=0.45\textwidth,trim={6cm 1cm 6cm 1cm},clip]{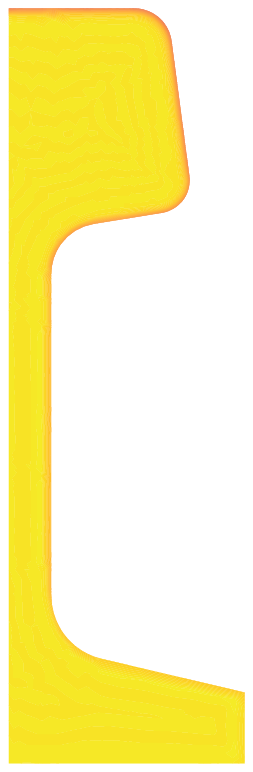}
\includegraphics[width=0.45\textwidth,trim={6cm 1cm 6cm 1cm},clip]{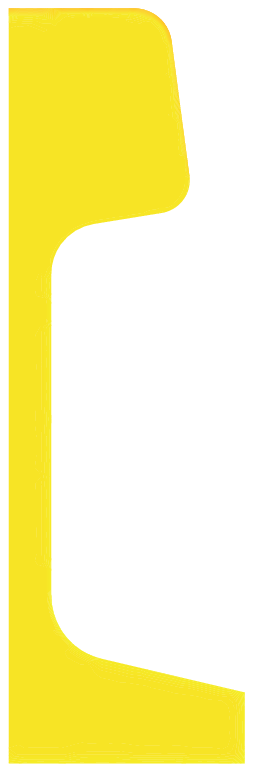}
\end{center}
    \caption{State 4}
  \end{subfigure}
  
   \begin{subfigure}[b]{0.22\textwidth}
    \begin{center}
\includegraphics[width=0.45\textwidth,trim={6cm 1cm 6cm 1cm},clip]{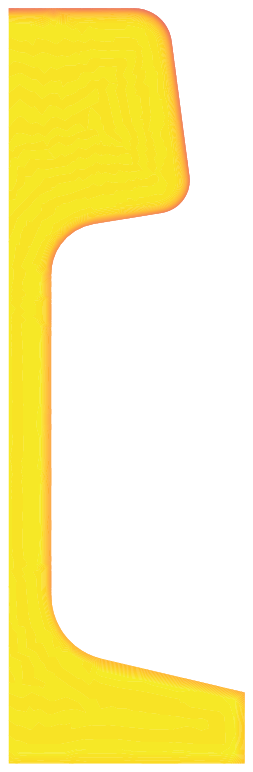}
\includegraphics[width=0.45\textwidth,trim={6cm 1cm 6cm 1cm},clip]{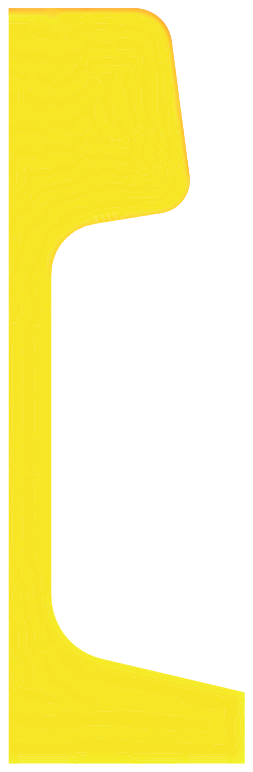}
\end{center}
    \caption{State 5}
  \end{subfigure}
  \hspace{0.7em}
  \begin{subfigure}[b]{0.22\textwidth}
    \begin{center}
\includegraphics[width=0.45\textwidth,trim={6cm 1cm 6cm 1cm},clip]{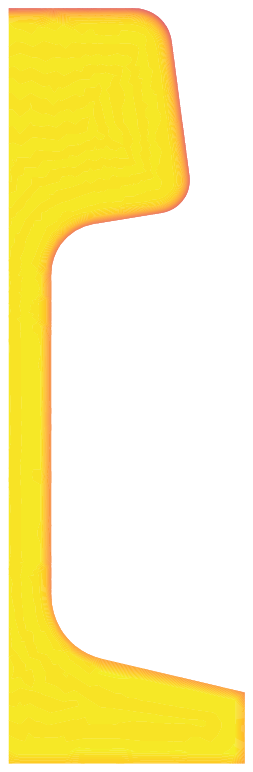}
\includegraphics[width=0.45\textwidth,trim={6cm 1cm 6cm 1cm},clip]{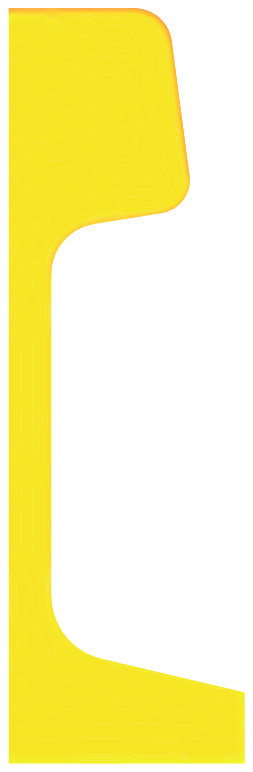}
\end{center}
    \caption{State 6}
  \end{subfigure}
  \hspace{0.7em}
  \begin{subfigure}[b]{0.22\textwidth}
    \begin{center}
\includegraphics[width=0.45\textwidth,trim={6cm 1cm 6cm 1cm},clip]{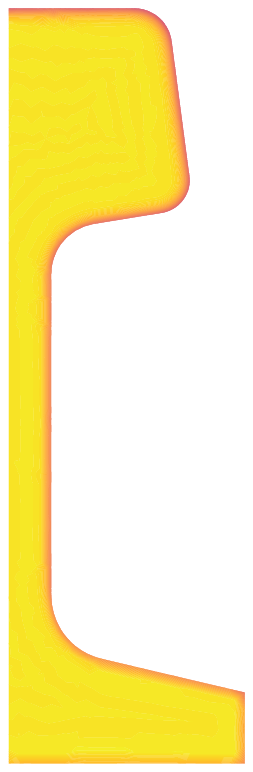}
\includegraphics[width=0.45\textwidth,trim={6cm 1cm 6cm 1cm},clip]{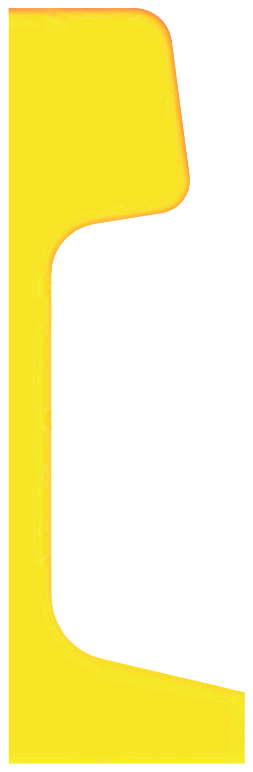}
\end{center}
    \caption{State 7}
  \end{subfigure}
  \hspace{0.7em}
  \begin{subfigure}[b]{0.22\textwidth}
    \begin{center}
\includegraphics[width=0.45\textwidth,trim={6cm 1cm 6cm 1cm},clip]{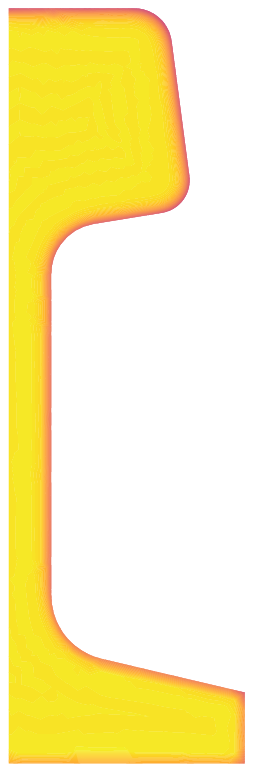}
\includegraphics[width=0.45\textwidth,trim={6cm 1cm 6cm 1cm},clip]{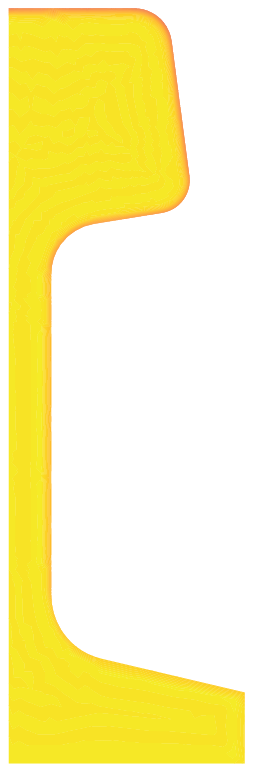}
\end{center}
    \caption{State 8}
  \end{subfigure}

   \begin{subfigure}[b]{0.22\textwidth}
    \begin{center}
\includegraphics[width=0.45\textwidth,trim={6cm 1cm 6cm 1cm},clip]{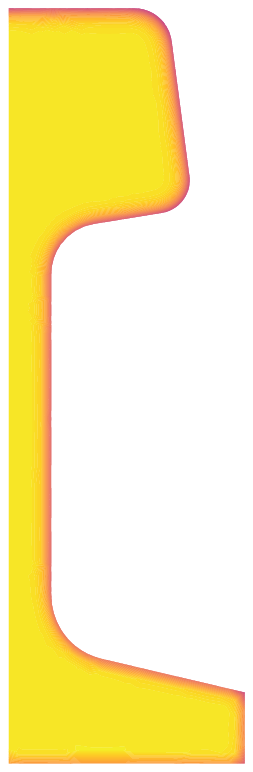}
\includegraphics[width=0.45\textwidth,trim={6cm 1cm 6cm 1cm},clip]{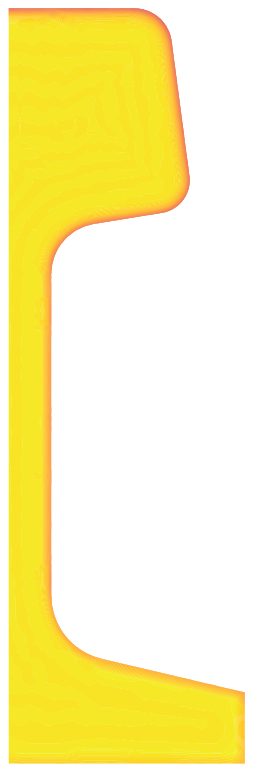}
\end{center}
    \caption{State 9}
  \end{subfigure}
  \hspace{0.7em}
  \begin{subfigure}[b]{0.22\textwidth}
    \begin{center}
\includegraphics[width=0.45\textwidth,trim={6cm 1cm 6cm 1cm},clip]{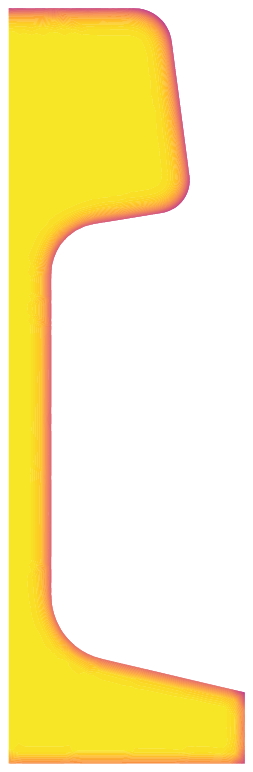}
\includegraphics[width=0.45\textwidth,trim={6cm 1cm 6cm 1cm},clip]{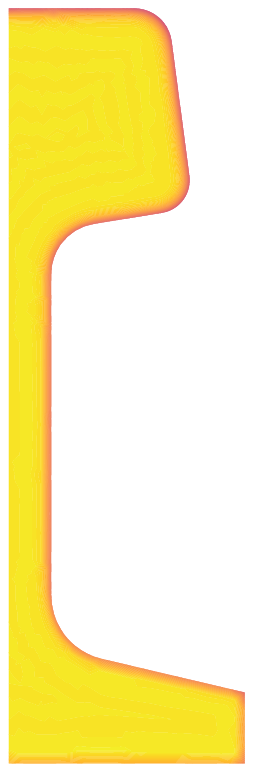}
\end{center}
    \caption{State 10}
  \end{subfigure}
  \hspace{0.7em}
  \begin{subfigure}[b]{0.22\textwidth}
    \begin{center}
\includegraphics[width=0.45\textwidth,trim={6cm 1cm 6cm 1cm},clip]{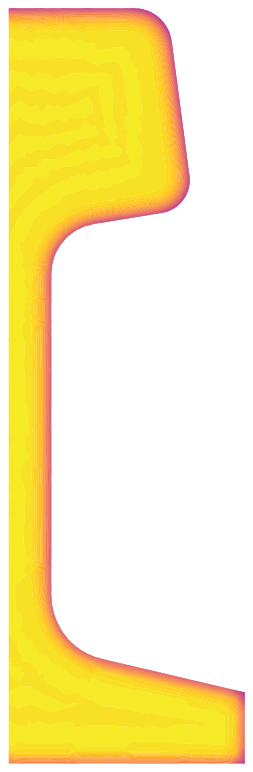}
\includegraphics[width=0.45\textwidth,trim={6cm 1cm 6cm 1cm},clip]{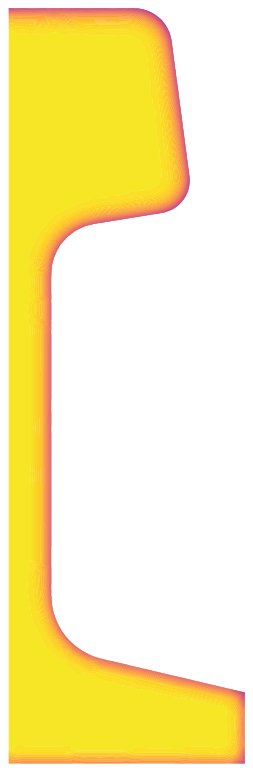}
\end{center}
    \caption{State 11}
  \end{subfigure}
  \hspace{0.7em}
  \begin{subfigure}[b]{0.22\textwidth}
    \begin{center}
\includegraphics[width=0.45\textwidth,trim={6cm 1cm 6cm 1cm},clip]{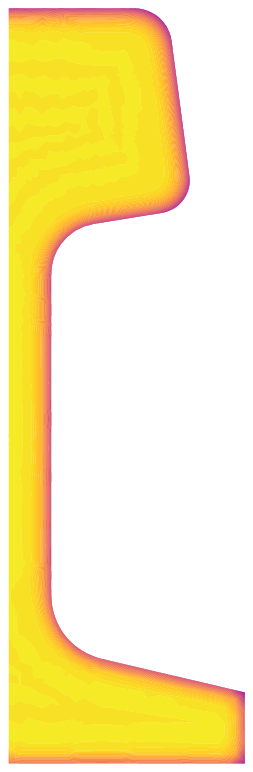}
\includegraphics[width=0.45\textwidth,trim={6cm 1cm 6cm 1cm},clip]{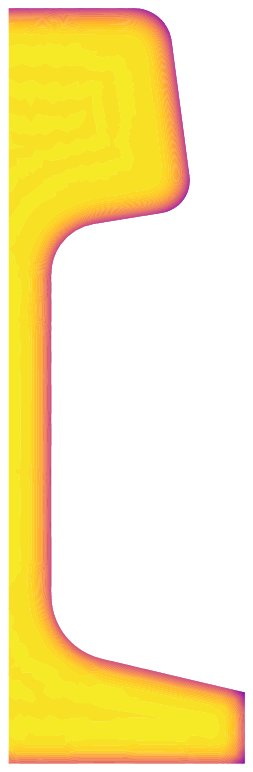}
\end{center}
    \caption{State 12}
  \end{subfigure}

   \begin{subfigure}[b]{0.22\textwidth}
    \begin{center}
\includegraphics[width=0.45\textwidth,trim={6cm 1cm 6cm 1cm},clip]{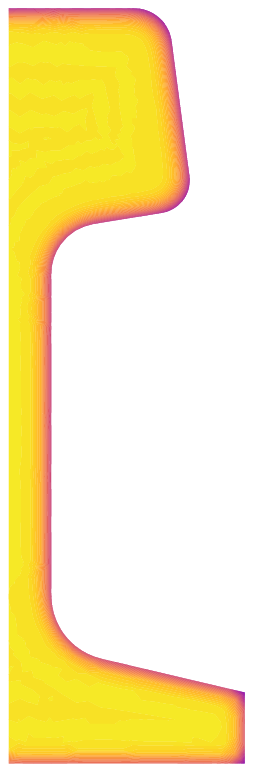}
\includegraphics[width=0.45\textwidth,trim={6cm 1cm 6cm 1cm},clip]{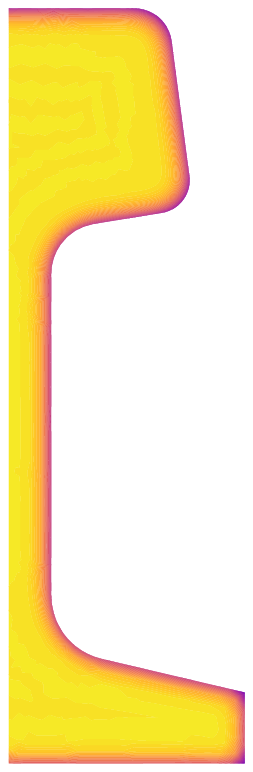}
\end{center}
    \caption{State 13}
  \end{subfigure}
  \hspace{0.7em}
  \begin{subfigure}[b]{0.22\textwidth}
    \begin{center}
\includegraphics[width=0.45\textwidth,trim={6cm 1cm 6cm 1cm},clip]{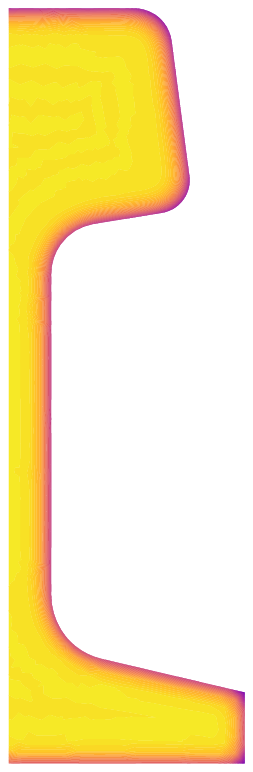}
\includegraphics[width=0.45\textwidth,trim={6cm 1cm 6cm 1cm},clip]{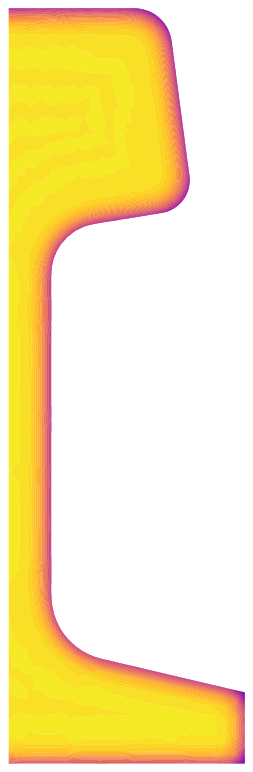}
\end{center}
    \caption{State 14}
  \end{subfigure}
  \hspace{-3em}
  \begin{subfigure}[b]{0.5\textwidth}
    \begin{center}
\includegraphics[width=0.2\textwidth,trim={6cm 1cm 6cm 1cm},clip]{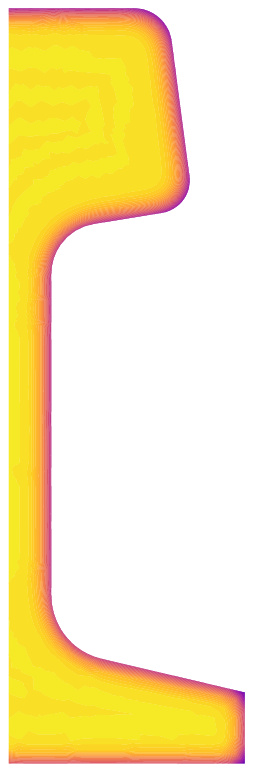}
\hspace{-1em}
\includegraphics[width=0.385\textwidth,trim={8cm 1cm 1cm 1cm},clip]{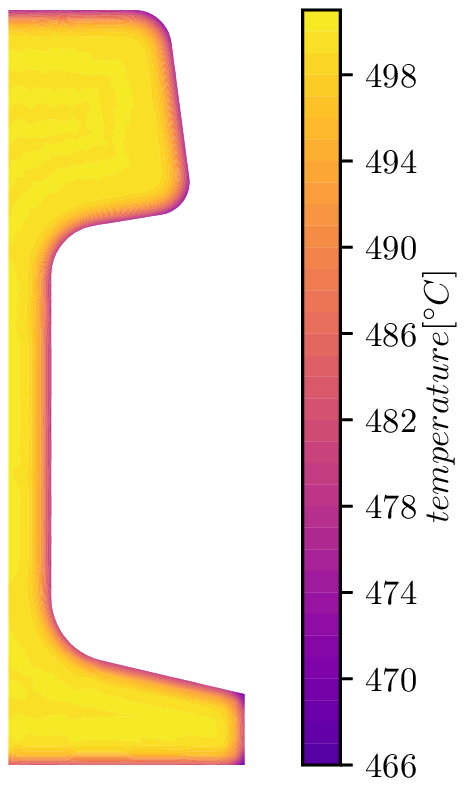}
\end{center}
    \caption{State 15}
    
  \end{subfigure}
  \caption{Cooling of steel profiles (Section~\ref{subsec:SteelRail}). High-dimensional states selected by sampling equidistant times (left) and by active operator inference (right) from the dictionary for $n=7$. For equidistant sampling, a majority of the states selected have cooler temperatures at the domain boundary with Robin condition. In contrast, active operator inference selects more states at the beginning of the cooling process, which leads to more accurate models in our experiments.}
  \label{fig:SteelProfile_DOE_States}
\end{figure}

\subsection{Diffusive Lotka-Volterra model for population dynamics of fish species} \label{subsec:DiffLotkaVolterra} Section~\ref{subsec:LotkaVolterraModel} discusses the model and the problem setup while Section~\ref{subsec:LotkaVolterraResults} summarizes the results of the numerical experiments.

\begin{figure}
  \begin{subfigure}[b]{0.32\textwidth}
    \begin{center}
\hspace*{-0.5cm}\includegraphics[width=1.3\textwidth]{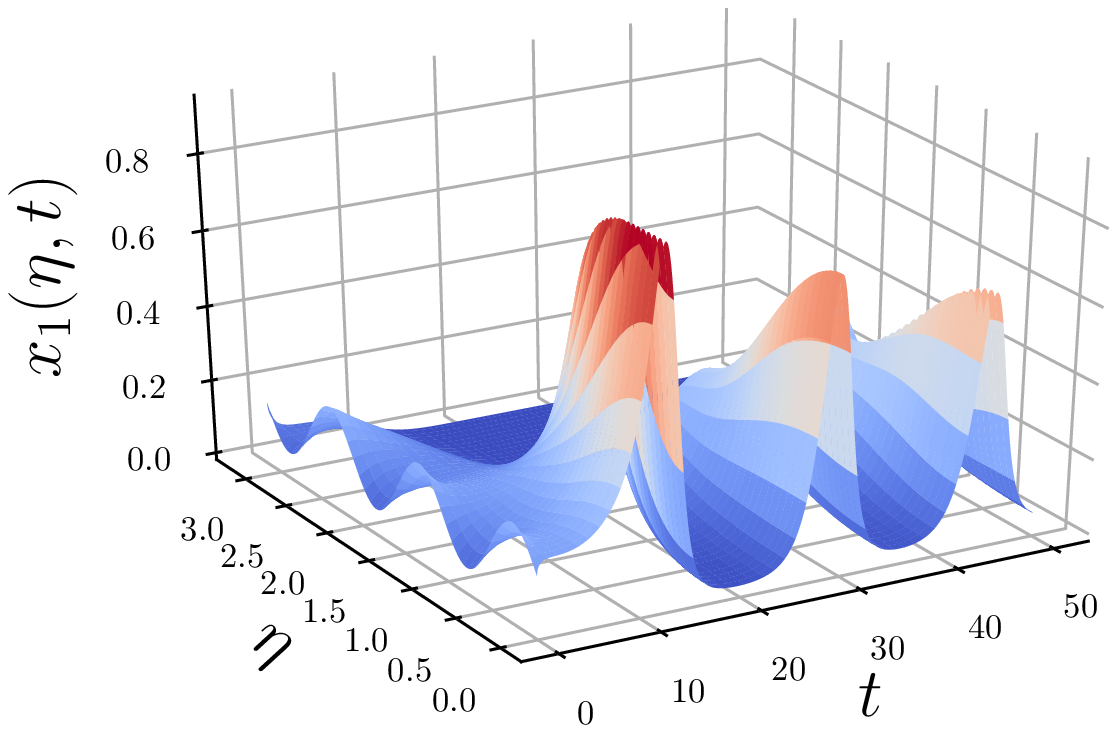}
\end{center}
  \end{subfigure}
  \begin{subfigure}[b]{0.32\textwidth}
   \begin{center}
\hspace*{-0.5cm}\includegraphics[width=1.3\textwidth]{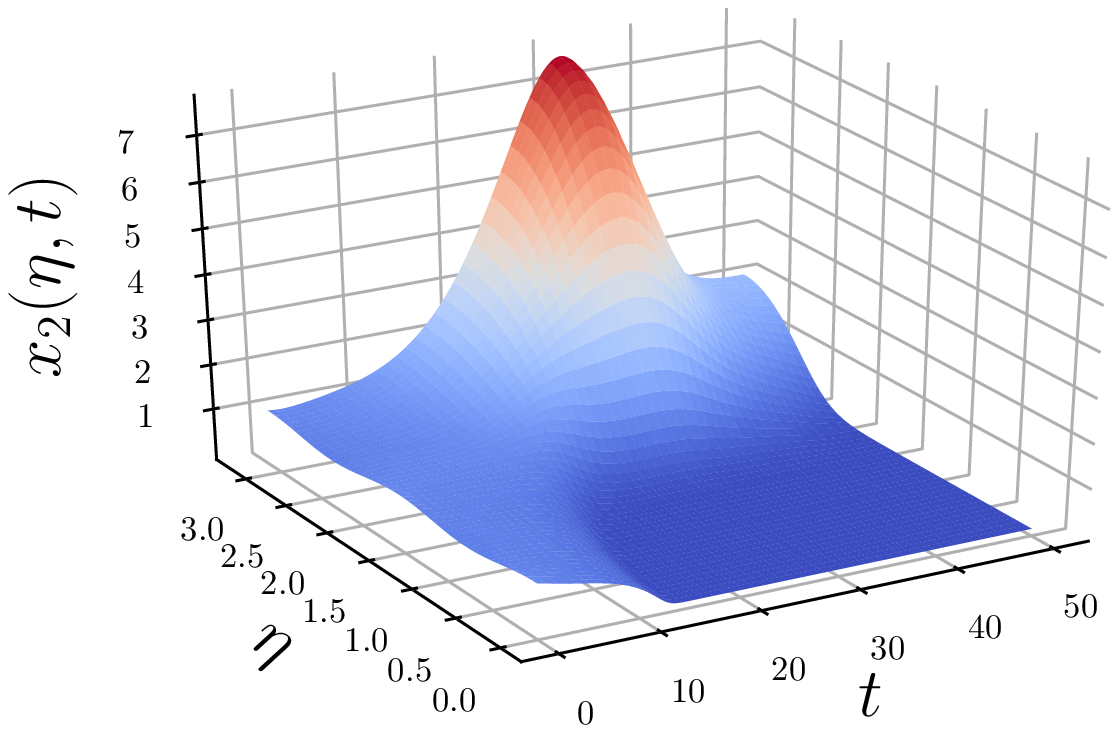}
\end{center}
  \end{subfigure}
  \begin{subfigure}[b]{0.32\textwidth}
  \begin{center}
\hspace*{-0.5cm}\includegraphics[width=1.3\textwidth]{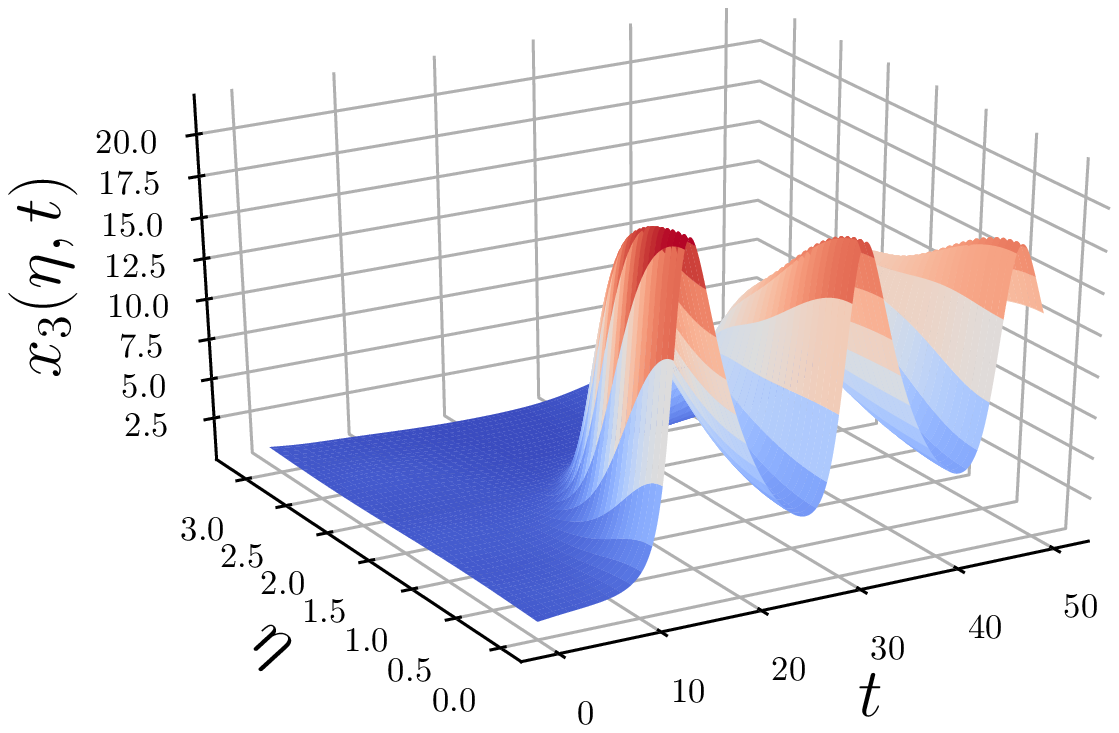}
\end{center}
  \end{subfigure}
  \caption{High-dimensional system trajectory of the population dynamics of fish species (Section~\ref{subsec:DiffLotkaVolterra}) with the test initial condition.}
  \label{fig:LotkaVolterra_FOM}
\end{figure}

\begin{figure}
  \begin{subfigure}[b]{0.45\textwidth}
    \begin{center}
{{\Large\resizebox{1.15\columnwidth}{!}{\input{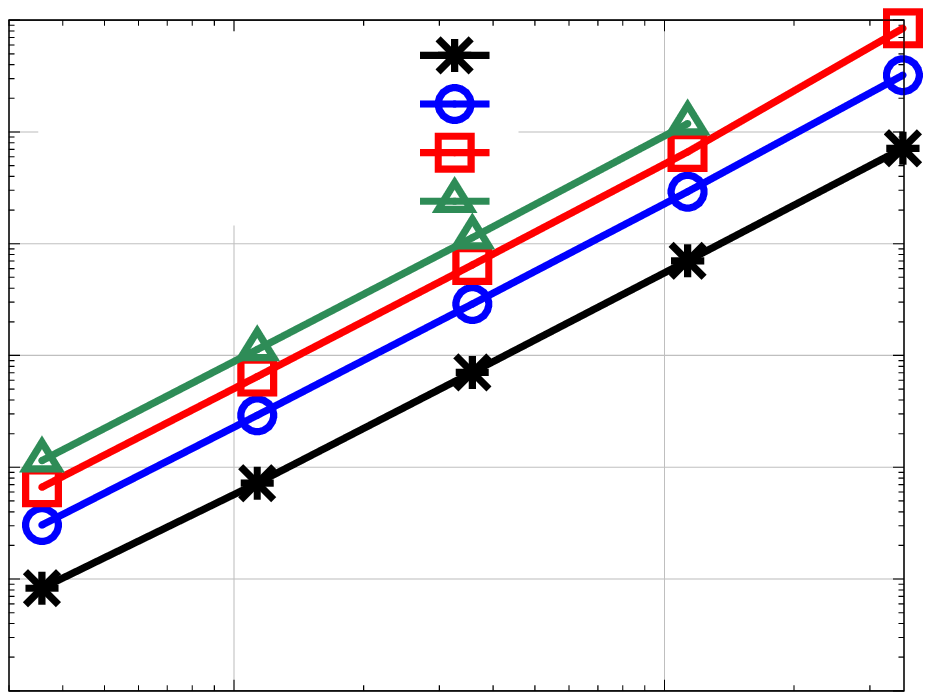}}}}
\end{center}
    \caption{dimension $n=12$}
  \end{subfigure}
  \hspace{1em}
  \begin{subfigure}[b]{0.45\textwidth}
   \begin{center}
{{\Large\resizebox{1.15\columnwidth}{!}{\input{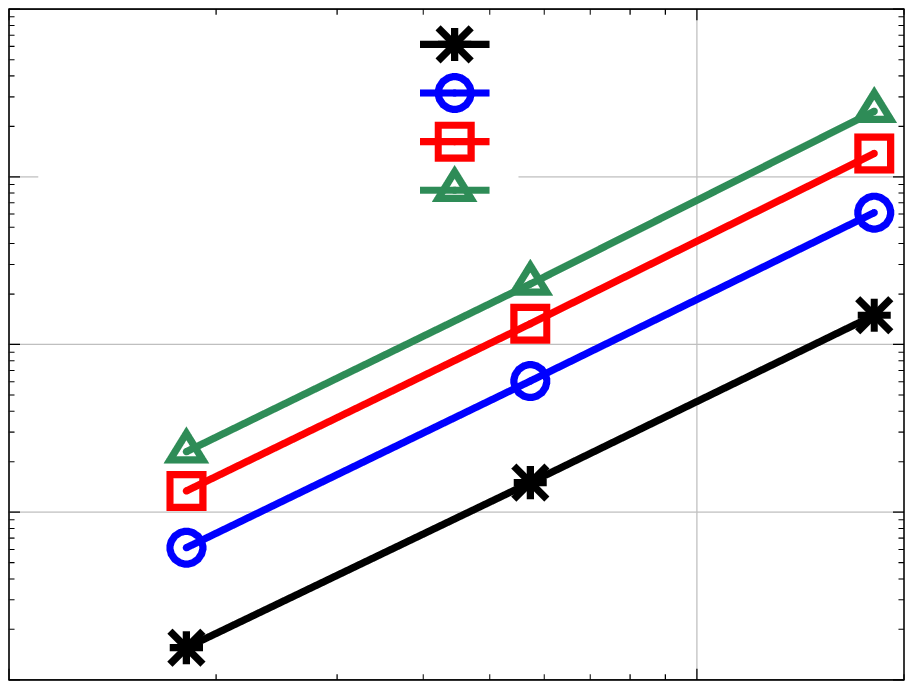}}}}
\end{center}
    \caption{dimension $n=15$}
  \end{subfigure}
  \caption{Population dynamics of fish species (Section~\ref{subsec:DiffLotkaVolterra}). For this quadratic system, an order of magnitude decay in the noise-to-signal ratio causes a two orders of magnitude decay in the estimated bias, demonstrating the bound of Proposition~\ref{prop:nonlinearpoly}.}
  \label{fig:LotkaVolterra_decay}
\end{figure}

\subsubsection{Model description} \label{subsec:LotkaVolterraModel}

Consider the population dynamics of three species of fish species in the Danube river \cite{Kmet1994}. At time $t > 0$ and distance $\eta$ from the mouth of the river, set $x_1(\eta,t), x_2(\eta,t), x_3(\eta,t)$ to be the density of forage fishes, German carp, and predators, respectively. For $\eta \in [0,\pi]$ and $t \in [0,T]$, a diffusive Lotka-Volterra model that describes the interaction between the species is given by 
    \begin{align} \label{eq:LotkaVolterraCont}
        \frac{\partial x_1 (\eta,t)}{\partial t} & = d_1 \frac{\partial^2 x_1(\eta,t)}{\partial \eta^2} + x_1(a_1 - a_2 x_2 - a_3 x_3 ) \\
        \frac{\partial x_2(\eta,t)}{\partial t} & = d_2 \frac{\partial^2 x_2(\eta,t)}{\partial \eta^2} + x_2(a_4 - a_5 x_3) \notag \\
        \frac{\partial x_3(\eta,t)}{\partial t} & = d_3 \frac{\partial^2 x_3 (\eta,t)}{\partial \eta^2} + x_3(a_6 x_1 + a_7 x_2 - a_8) \notag
    \end{align}
    subject to the Neumann boundary condition $\frac{\partial x_i (0,t)}{\partial \eta} = \frac{\partial x_i (\pi,t)}{\partial \eta} = 0$  for $i=1,2,3$. The values of the constants are $a_1 = 1.01, a_2 = 0.93, a_3 = 0.1, a_4 = 0.19, a_5 = 0.2, a_6 = 1, a_7 = 0.05, a_8 = 0.2, d_1 = 0.01, d_2 = 0.03, d_3 =0.009$. 
    
The differential equation \eqref{eq:LotkaVolterraCont} is spatially discretized at 100 equidistant points in $[0,\pi]$. To temporally discretize \eqref{eq:LotkaVolterraCont}, we apply the Crank-Nicolson finite difference scheme to the diffusion term and evaluate the nonlinear term explicitly in time with step size $\delta t = 0.01$, resulting in an implicit-explicit scheme. This leads to the autonomous discrete system \eqref{eq:NoisyHighDimSys} and \eqref{eq:PolySys} with $\ell=2$ where $\bfx_k \in \R^N, N =300$.
 
Set $x_3^* = a_4/a_5$, $x_2^* =  (a_1 a_5 - a_3 a_4)/(a_2 a_5)$, $x_1^* =  (a_8 - a_7x_2^*)/a_6$. Observe that $(x_1,x_2,x_3) = (x_1^*,x_2^*,x_3^*) $ is a spatially homogeneous equilibrium point of \eqref{eq:LotkaVolterraCont}. The basis matrix $\bfV$ is obtained from snapshots $\bfx_k^{\text{basis}}$ of the high-dimensional system initiated at the following 6 conditions $x_{1,i}^{\text{basis}}(\eta,0) = x_1^* + \gamma_{1i}\sin(6 \gamma_{2i} \eta)/10, x_{2,i}^{\text{basis}}(\eta,0)=  x_2^* + \gamma_{3i} \cos(4 \gamma_{4i} \eta)/10, x_{3,i}^{\text{basis}}(\eta,0) =  x_3^* + \gamma_{5i} \sin(2 \gamma_{6i} \eta)/10, i = 1,\dots,6$, where for each $i$, $\gamma_{1i},\dots,\gamma_{6i}$ are realizations of a uniform random variable on $[0,1]$. For each initial condition, the high-dimensional system is simulated until $T=50$ resulting in 30000 elements in $\Dcal$. These initial states represent perturbations around the spatially homogeneous equilibrium. The standard deviations of the noise $\sigma$ are 5 equidistant values in the logarithm scale between $1\times 10^{-4}$ and $1\times 10^{-2}$. For prediction, the initial condition we use is given by $x_1^{\text{test}}(\eta,0) = x_1^* + \sin(6\eta)/10$, $x_2^{\text{test}}(\eta,0) = x_2^* + \cos(4 \eta)/10$, and $x_3^{\text{test}}(\eta,0) = x_3^* + \sin(2\eta)/10$. Figure~\ref{fig:LotkaVolterra_FOM} shows the high-dimensional state trajectories $x_1^{\text{test}}(\eta,t), x_2^{\text{test}}(\eta,t), x_3^{\text{test}}(\eta,t)$ for $t \in [0,50].$

\subsubsection{Results} \label{subsec:LotkaVolterraResults}

\begin{figure}
  \begin{subfigure}[b]{0.45\textwidth}
    \begin{center}
{{\Large\resizebox{1.15\columnwidth}{!}{\input{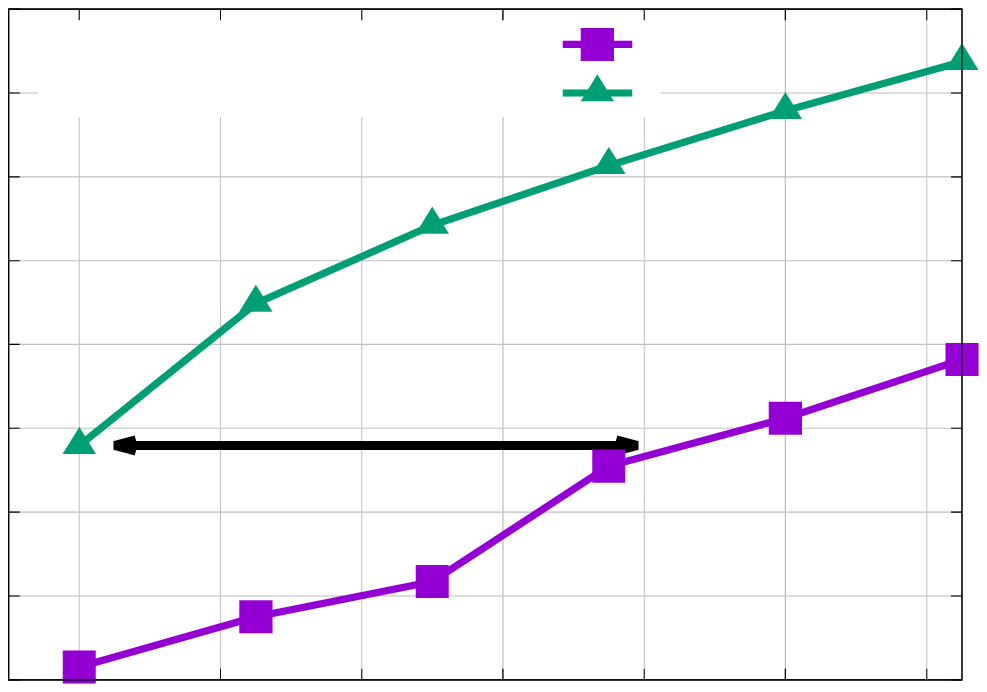}}}}
\end{center}
    \caption{dimension $n=12$}
  \end{subfigure}
  \hspace{1em}
  \begin{subfigure}[b]{0.45\textwidth}
   \begin{center}
{{\Large\resizebox{1.15\columnwidth}{!}{\input{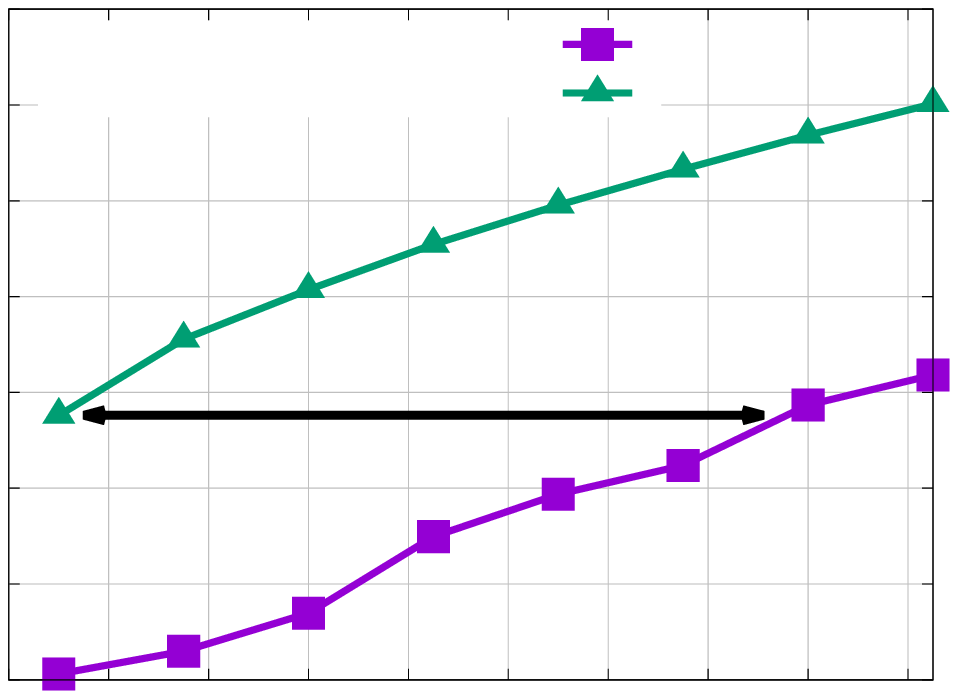}}}}
\end{center}
    \caption{dimension $n=15$}
  \end{subfigure}
  \caption{Population dynamics of fish species (Section~\ref{subsec:DiffLotkaVolterra}). Active operator inference requires up to 2 times fewer queries to the high-dimensional system for generating data than a traditional equidistant-in-time sampling process.}
  \label{fig:LotkaVolterra_DOE_singval}
\end{figure}

A low-dimensional model is inferred from noisy data for dimensions $n \in \{12,15\}$. Active operator inference is applied to select 100 rows for $n = 12$, leading to a data matrix with $\singValMin(\bfD) = 0.2794$. For $n=15$, 150 rows are selected, which results in $\singValMin(\bfD)=0.0552$.  Monte Carlo estimates of the bias and MSE are shown in Figure~\ref{fig:LotkaVolterra_decay} and Figure~\ref{fig:LotkaVolterra_MSE}, respectively. The number of Monte Carlo samples used is $5 \times 10^7$. The plots are consistent with the analysis in Proposition~\ref{prop:nonlinearpoly} and \ref{prop:nonlinearpolyMSE}, particularly for quadratic systems, since we observe that an order decay in the noise-to-signal ratio leads to two orders decay in the estimated bias and MSE. The missing value in Figure~\ref{fig:LotkaVolterra_decay}(a) represents a large bias in $\hbfx_k^{\text{test}}$ which we do not plot and is caused by the accumulation of errors in the learned reduced operators over time.  It represents a non-asymptotic regime in which the constants in the bias dominate the behavior of the noise-to-signal ratio. In Figure~\ref{fig:LotkaVolterra_decay}(b) and \ref{fig:LotkaVolterra_MSE}, results for larger values of $\sigma$ are not shown in the plot for the same reason.

We now compare active operator inference to a traditional equidistant-in-time sampling from the dictionary. The minimum singular value of the data matrix resulting from both approaches is compared in Figure~\ref{fig:LotkaVolterra_DOE_singval}. In this example, active operator inference reduces the number of times the high-dimensional system is queried by up to roughly a factor of two. The estimated bias and the MSE for both approaches at $\sigma = 1\times 10^{-3}$ is shown in Figure~\ref{fig:LotkaVolterra_DOE} and in the right panel of Figure \ref{fig:LotkaVolterra_MSE}. We also plot the estimated MSE at $T=50$ using 10 Monte Carlo samples for $n=15$ in the bottom panel of Figure \ref{fig:LotkaVolterra_MSE}.  Results are not plotted if the corresponding models numerically led to unstable behavior with unbounded errors. Active operator inference provides reasonable numerical predictions in all cases, whereas equidistant sampling quickly leads to models that show unstable behavior. This behavior is amplified for increasing dimension $\nr$. Overall, the results indicate that for polynomially nonlinear systems it becomes even more important than for linear systems to carefully query the high-dimensional system.

\begin{figure}
  \begin{subfigure}[b]{0.45\textwidth}
    \begin{center}
{{\Large\resizebox{1.15\columnwidth}{!}{\input{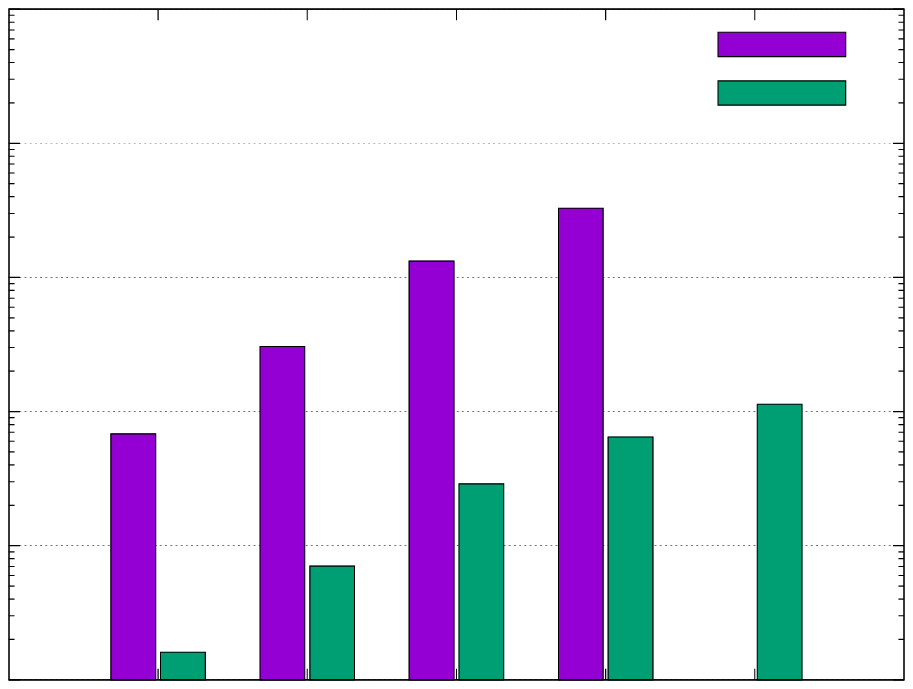}}}}
\end{center}
    \caption{dimension $n=12$}
  \end{subfigure}
  \hspace{1em}
  \begin{subfigure}[b]{0.45\textwidth}
   \begin{center}
{{\Large\resizebox{1.15\columnwidth}{!}{\input{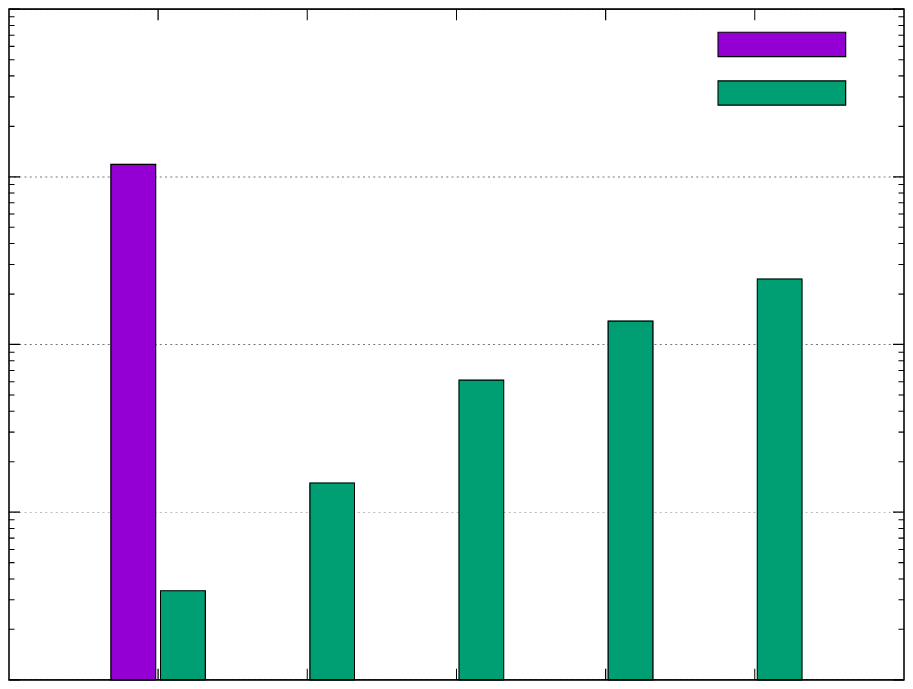}}}}
\end{center}
    \caption{dimension $n=15$}
  \end{subfigure}
  \caption{Population dynamics of fish species (Section~\ref{subsec:DiffLotkaVolterra}). Selecting the data matrix by sampling equidistant in time quickly leads to numerical instabilities in the learned models while the selection obtained with active operator inference leads to models that show stable and accurate behavior in this example.}
  \label{fig:LotkaVolterra_DOE}
\end{figure}

\begin{figure}
  \begin{subfigure}[b]{0.45\textwidth}
    \begin{center}
{{\Large\resizebox{1.15\columnwidth}{!}{\input{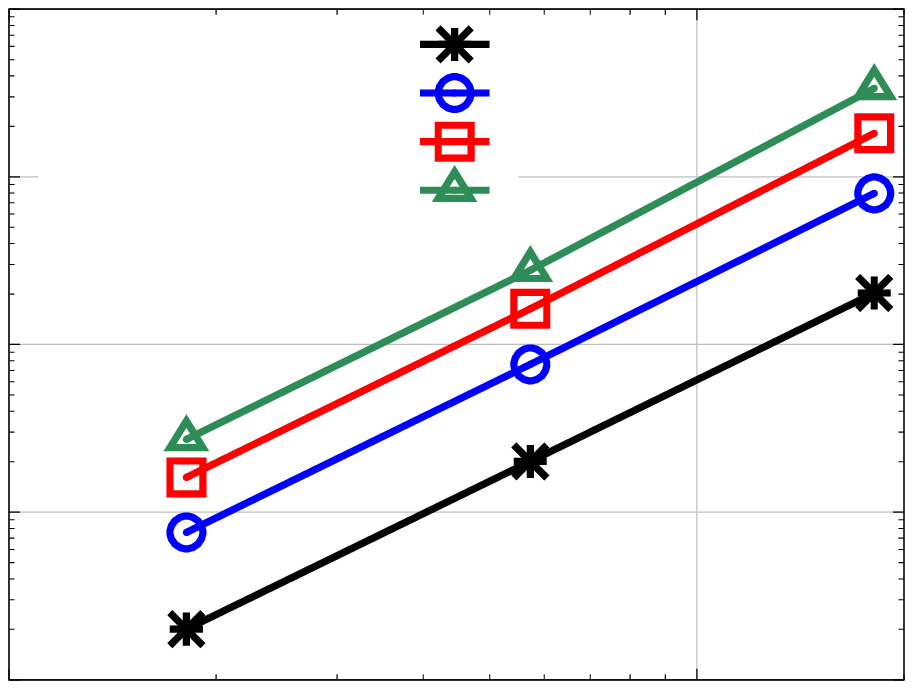}}}}
\end{center}
    \caption{decay of MSE, $n=15$}
  \end{subfigure}
  \hspace{1em}
  \begin{subfigure}[b]{0.45\textwidth}
   \begin{center}
{{\Large\resizebox{1.15\columnwidth}{!}{\input{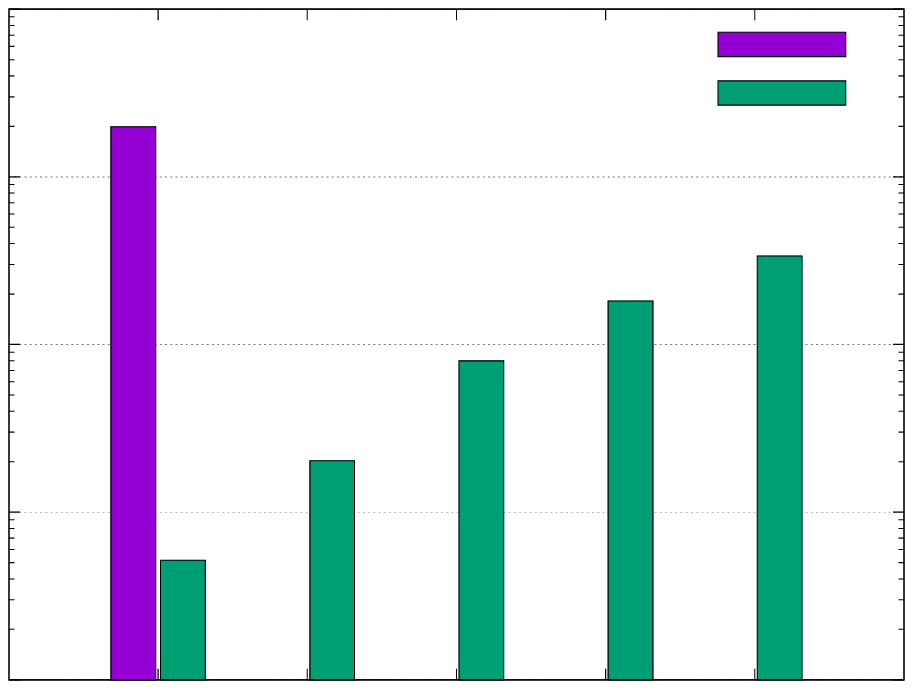}}}}
\end{center}
    \caption{equidistant vs Active OpInf, $n=15$}
  \end{subfigure}
 \begin{center}
  \begin{subfigure}[b]{0.45\textwidth}
    \begin{center}
{{\Large\resizebox{1.15\columnwidth}{!}{\input{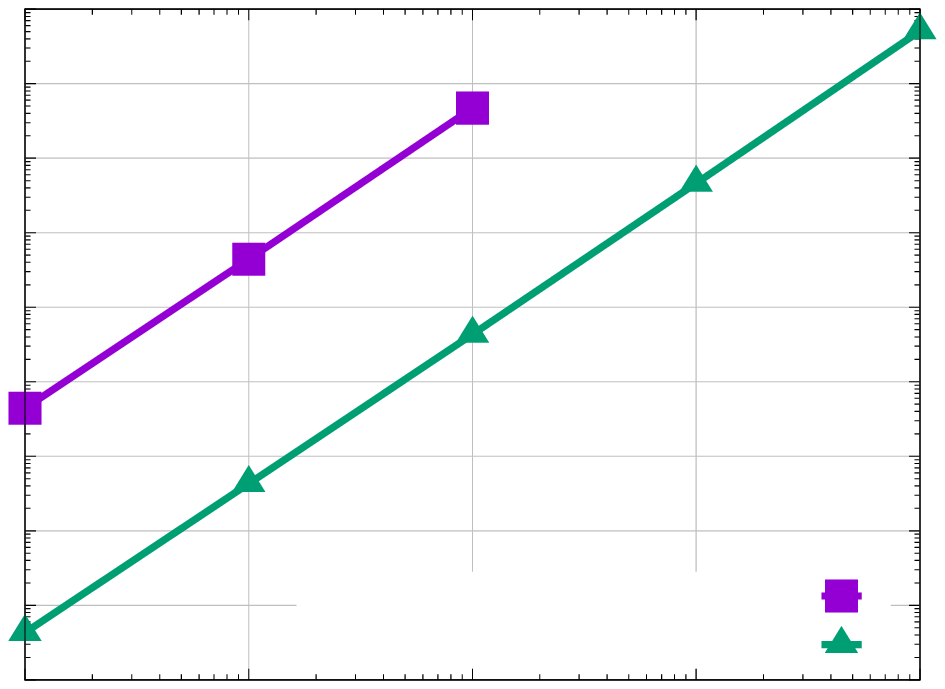}}}}
\end{center}
    \caption{time step 5000}
  \end{subfigure}
 \end{center}
  \caption{Population dynamics of fish species (Section~\ref{subsec:DiffLotkaVolterra}). An order of magnitude decay in the noise-to-signal ratio leads to a two orders of magnitude decay in the estimated MSE, which is in agreement with Proposition \ref{prop:nonlinearpolyMSE}. Sampling the dictionary at equidistant times results in learned models that become numerically unstable while active operator inference leads to models with orders of magnitude lower MSEs.}
  \label{fig:LotkaVolterra_MSE}
\end{figure}

\section{Conclusions}
In this work, we established probabilistic guarantees on predictions made with low-dimensional models learned from noisy data, which motivated an active data sampling approach to reduce the effect of noise. The key ingredient of the analysis and the numerical approach was building a bridge from data-driven modeling via operator inference and re-projection to classical projection-based model reduction. Thus, the proposed approach can be seen as an example of scientific machine learning that demonstrates the benefits of merging traditional scientific computing concepts such as model reduction with learning methods to effectively leverage data.

\section*{Acknowledgements}
We are grateful to Jens Saak for providing the code to generate the computational mesh and the system matrices for the heat transfer problem on steel profiles.  We also thank Jonathan Niles-Weed for directing us to references for deriving upper bounds on moments of the norm of Gaussian random matrices.

This work was partially supported by US Department of Energy, Office of Advanced Scientific Computing
Research, Applied Mathematics Program (Program Manager Dr. Steven Lee), DOE Award DESC0019334,
and by the National Science Foundation under Grant DMS-2012250.

\bibliographystyle{abbrv}
\bibliography{references.bib} 

\appendix
\section{Matrix concentration inequalities} \label{appendix}

    \begin{theorem} [Theorem 5.32 and Proposition 5.34 in \cite{paper:Vershynin2012}]
            Let $\bfA$ be an $N \times n$ matrix whose entries are independent standard normal random variables. Then
            \begin{align} \label{eq:NormGaussBoundExpectation}
                \Exp[\|\bfA\|_2] \le \sqrt{N} + \sqrt{n}
            \end{align}
            and for $t \ge 0$,
            \begin{align} \label{eq:MatrixConcVershynin}
                P\left (\biggl|\|\bfA\|_2 - \Exp[\|\bfA\|_2]\biggr| > t \right) \le 2e^{-t^2/2}.
            \end{align}
          \end{theorem}
          
\section{Proof of Lemma~\ref{lemma:GaussNormPowerBound}} \label{appendix:GaussNormPowerBoundProof}

By using the triangle inequality for the norm $\Exp[|\cdot|^l]^{1/l}$,
    \begin{align} \label{eq:GaussMatrixLemmaTriangleIneq}
        \left(\Exp[\|\bfG\|_2^l]\right)^{1/l} & = \left(\Exp \left[ \bigl|\|\bfG\|_2 -\Exp[\|\bfG\|_2] +\Exp[\|\bfG\|_2] \bigr|^l \right] \right)^{1/l} \notag \\
        & \le \left(\Exp \left[ \bigl |\|\bfG\|_2 -\Exp[\|\bfG\|_2] \bigr |^l \right] \right)^{1/l} +\Exp[\|\bfG\|_2] \notag \\
        & \le \left(\Exp \left[ \bigl | \|\bfG\|_2 -\Exp[\|\bfG\|_2] \bigr |^l \right] \right)^{1/l} + \sqrt{n} + \sqrt{p}
    \end{align}
    where we have used the bound \eqref{eq:NormGaussBoundExpectation}.
    
    Denote by $\Gamma(\cdot)$ the gamma function. Recall that that $\Gamma(x+1) \le x^x$ for $x \ge 0$ and $\Gamma(x+1) = x \Gamma(x)$. To bound the first term in the right hand side of the inequality \eqref{eq:GaussMatrixLemmaTriangleIneq}, we proceed as follows. For $t \ge 0$, 
    \begin{align*}
        \Exp \left[ \bigl |\|\bfG\|_2 -\Exp[\|\bfG\|_2] \bigr |^l \right] & = l \int_0^{\infty} t^{l-1} P\left( \biggl | \|\bfG\|_2  -\Exp[\|\bfG\|_2] \biggr| \ge t \right) \,dt \\
        &\le   2 l  \int_0^{\infty} t^{l-1} e^{-t^2/2} \,dt  = 2 l \int_0^{\infty} (2u)^{\frac{l-2}{2}} e^{-u} \,du \notag \\
        & = l 2^{l/2} \int_0^{\infty} u^{l/2 - 1} e^{-u} \,du = 2^{l/2 +1} \frac{l}{2} \Gamma \left(\frac{l}{2}\right ) = 2^{l/2 + 1} \Gamma \left( \frac{l}{2} +1 \right) \notag \\
        & \le 2^{l/2 + 1} \left( \frac{l}{2} \right)^{l/2} = 2l^{l/2} \notag
    \end{align*}
    where we utilized the concentration inequality \eqref{eq:MatrixConcVershynin} in Appendix~\ref{appendix} and properties of the gamma function mentioned above.
    
    This implies that $\left(\Exp \left[ \left|\|\bfG\|_2 -\Exp[\|\bfG\|_2] \right|^l \right] \right)^{1/l} \le 2^{1/l}\sqrt{l}$ and the conclusion follows from \eqref{eq:GaussMatrixLemmaTriangleIneq}.

\end{document}